\def\ifundefined#1{\expandafter\ifx\csname#1\endcsname\relax}
\let\arXiv = 1

\ifundefined{arXiv}
    \documentclass[journal]{IEEEtran}
\else
    \documentclass[10pt]{article}
\fi

\usepackage[english]{babel}
\usepackage{booktabs}

\usepackage{ifpdf}

\usepackage{cite} 
\usepackage{url}
\usepackage{hyperref}

\usepackage[pdftex]{graphicx}
\graphicspath{{./figures/}}
\ifundefined{arXiv}
\else
    \usepackage{IEEEtrantools}
    \usepackage{geometry}[margin = 1in]
    \usepackage{setspace}
\fi
\usepackage{color}

\usepackage{pgf, tikz, pgfplots}
\usetikzlibrary{shapes, arrows, automata}
\usetikzlibrary{calc,hobby,decorations}

\usepackage[cmex10]{amsmath}
\usepackage{amsfonts, amssymb, amsthm}
\usepackage{mathrsfs}

\usepackage{algorithm} 
\usepackage{algorithmic}
\usepackage{enumerate}
\usepackage{multirow}
\usepackage{rotating}
\usepackage{subcaption}
\usepackage{needspace}

\newcommand{\myparagraph}[1]{\needspace{1\baselineskip}\medskip\noindent {\bf #1}}

\input{auxFiles/mySymbol.sty}
\input{auxFiles/pennColors.sty}

\def\Tr{\mathsf{T}}

\def\eps{\varepsilon}

\def\SO3{\mathrm{SO(3)}}

\newtheorem{assumption}{\hspace{0pt}\bf Assumption \hspace{-0.15cm}}
\newtheorem{lemma}{\hspace{0pt}\bf Lemma}

\newtheorem{theorem}{\hspace{0pt}\bf Theorem}

\newtheorem{remark}{\hspace{0pt}\bf Remark}

\newtheorem{definition}{\hspace{0pt}\bf Definition}

\renewcommand{\blue}{}

\begin{document}

\title{Wide and Deep Graph Neural Network with Distributed Online Learning}

\author{Zhan~Gao,~
    Fer\hspace{0.015cm}nando Gama,~
    and~Alejandro~Ribeiro\vspace{-6mm}
    \thanks{This work is supported by NSF CCF 1717120, ARO W911NF1710438, ARL DCIST CRA W911NF-17-2-0181, ISTC-WAS and Intel DevCloud. Z. Gao and A. Ribeiro are with the Dept. of Electrical and Systems Eng., Univ. of Pennsylvania. F. Gama is with the Electrical Comput. Eng., Rice Univ.. Email: \{gaozhan,aribeiro\}@seas.upenn.edu, fgama@rice.edu. Part of the results in this paper appear in \cite{Gao2021-WideDeep}.
    }
}

\markboth{IEEE TRANSACTIONS ON SIGNAL PROCESSING (SUBMITTED)}%
{Wide and Deep Graph Neural Networks with Distributed Online Learning}

\maketitle

\begin{abstract}
Graph neural networks (GNNs) are naturally distributed architectures for learning representations from network data. This renders them suitable candidates for decentralized tasks. In these scenarios, the underlying graph often changes with time due to link failures or topology variations, creating a mismatch between the graphs on which GNNs were trained and the ones on which they are tested. Online learning can be leveraged to retrain GNNs at testing time to overcome this issue. However, most online algorithms are centralized and usually offer guarantees only on convex problems, which GNNs rarely lead to. This paper develops the Wide and Deep GNN (WD-GNN), a novel architecture that can be updated with distributed online learning mechanisms. The WD-GNN consists of two components: the wide part is a linear graph filter and the deep part is a nonlinear GNN. At training time, the joint wide and deep architecture learns nonlinear representations from data. At testing time, the wide, linear part is retrained, while the deep, nonlinear one remains fixed. This often leads to a convex formulation. We further propose a distributed online learning algorithm that can be implemented in a decentralized setting. We also show the stability of the WD-GNN to changes of the underlying graph and analyze the convergence of the proposed online learning procedure. Experiments on movie recommendation, source localization and robot swarm control corroborate theoretical findings and show the potential of the WD-GNN for distributed online learning.
\end{abstract}

\ifundefined{arXiv}
    \IEEEpeerreviewmaketitle
\else
\fi


\section{Introduction} \label{sec:intro}



Graph neural networks (GNNs) \cite{Scarselli2009, Ruiz2021-GNNs} are processing architectures that learn nonlinear representations from network data in a wide array of tasks ranging from abstract graphs, such as citation networks, recommendation systems and authorship attribution \cite{Gao2021-Training, Shchur2018-Pitfalls, Xu2019-How}, to physical graphs, including source localization, wireless communications, and robotics \cite{Gao2021-Variance, Gao2020-Resource, Gama2021-ControlGNN}. GNNs are capable of levaraging structural information present in graph signals to extract meaningful features. GNNs consist of a cascade of blocks, each one applying a graph convolution followed by a pointwise nonlinearity \cite{Bruna2014-SpectralGNN, Defferrard2016-ChebNets, Gama2019-Archit}. 

One of the key properties of GNNs is that they are local and distributed. They are local since they only require information from neighboring nodes and distributed since each node can compute its own output without a centralized unit. While seminal for decentralized learning \cite{Gama2021-ControlGNN}, their performance relies heavily on the structure of the underlying graph inherent in data \cite{Gama2020-Stability, gao2021stability, Ruiz2020-Transferability}. Oftentimes, problems of interest exhibit changes to the data structure between training and testing phases or involve dynamic systems \cite{frahling2008sampling, helwa2017multi, Gama2019-LinearControl}. For example, in the case of robot swarm coordination, the graph is determined by the communication network between robots which is, in turn, determined by their physical proximity. If robots move, the communication links will change and the graph structure will change as well. In this context, we may need adapting to the new data structure to maintain performance. GNNs have been shown to be resilient to mild structure changes, referred to as permutation equivariance and perturbation stability \cite{Gama2020-Stability, Lerman2020-Scattering}. While these properties guarantee certain transference \cite{Ruiz2020-Transferability}, the performance gets degraded inevitably under substantial structure changes. \blue{The work in \cite{Gao2020-Stochastic} assumed the graph dynamics as the random edge sampling (RES) model and developed stochastic graph neural networks (SGNNs) that account for the latter during training to improve the architecture robustness during testing. However, the RES model cannot represent all graph dynamics and the SGNN is limiting in certain applications.} In this paper, we propose to mitigate this performance degradation by leveraging online learning approaches \cite{bottou2004large, anderson2008theory}.

Online learning is a well-established paradigm for time-varying optimization problems, which has been shown successful in the fields of machine learning and signal processing \cite{dabbagh2005online}. In particular, online learning focuses on optimization problems that change continuously across time. At each time instance, online algorithms tackle a new instantiated optimization problem and perform a series of updates on the solution obtained at a previous time instance. In a nutshell, these algorithms generate a sequence of approximate optimizers that track the true ones of the time-varying optimization problem.

In order to leverage online learning in GNNs, we face two major roadblocks. First, optimality bounds and convergence guarantees are given only for convex problems \cite{shalev2012online}. Second, current online algorithms assume a centralized approach. This is particularly problematic since it violates the local and distributed nature of GNNs which is a key property of learning in physical systems \cite{Gama2021-ControlGNN}. Online learning has been investigated in designing neural networks (NNs) for dynamically varying problems. Specifically, the works in \cite{Sanz2012, Li2004} develop online algorithms for fully-connected neural networks (multi-layer perceptrons; MLP) with applications in dynamical condition monitoring and aircraft control. More recently, the authors in \cite{Hong2015, Molchanov2016} applied online learning in convolutional neural networks (CNNs) for visual tracking, detection and classification. While these works adopt online algorithms in NNs, theoretical analysis on the convergence of these algorithms is not presented, except for \cite{Ho20102} that proves the convergence of certain online algorithms for radial neural networks only.

This paper puts forth a novel architecture called the Wide and Deep Graph Neural Network (WD-GNN). This architecture is tailored for distributed tasks and, fundamentally, adds the capability of distributed online retraining. The WD-GNN consists of two components, a wide part consisting of a linear graph filter, and a deep part consisting of a nonlinear GNN. Both parts are trained jointly in an offline phase. During the online execution, the WD-GNN allows for distributed re-learning with provable convergence guarantees. More in detail, our main contributions are as follows:

\begin{list}{}{ 
                 \setlength{\labelwidth}{20pt} 
                 \setlength{\leftmargin}{23pt} 
                 \setlength{\labelsep}{3pt} 
                 \setlength{\itemsep}{5pt}
                 \setlength{\topsep}{5pt}
                 \setlength{\parskip}{0pt}
              }

    \item[\textbf{(C1)}] \emph{The WD-GNN (Section \ref{sec:wideDeepGNN})}: We propose the WD-GNN consisting of two components, a wide part which is a linear graph filter and a deep part which is a nonlinear GNN. This combined architecture not only improves the representational capability but also inherits a distributed implementation for decentralized learning.
    
    \item[\textbf{(C2)}] \emph{Distributed online learning (Section \ref{sec:onlineLearning})}: We propose a learning procedure consisting of two phases: the offline training and the online execution and retraining. The offline phase trains the joint architecture to learn nonlinear representations from data. The online phase fixes the nonlinear deep part and only retrains the linear wide part, adapting to changing problem scenarios at testing time. We further propose a distributed online algorithm that successfully performs the online phase in a completely distributed manner.
    
    \item[\textbf{(C3)}] \emph{Stability analysis (Section \ref{sec:stability})}: We analyze the stability of the WD-GNN to perturbations on the underlying graph. We prove that the output difference of the WD-GNN induced by the graph perturbation is upper bounded proportionally to the perturbation size. The proportionality depends on the graph filter, the nonlinearity, and the architecture width and depth. This result indicates that the WD-GNN maintains performance under mild changes of the problem scenario. 
    
    \item[\textbf{(C4)}] \emph{Convergence analysis (Section \ref{sec:convergence})}: We provide convergence analysis for the proposed online learning procedure, which estimates the convergence rate and characterizes its dependence on the graph structure, algorithm parameters, problem setting and correlations between time-varying instances of the problem. The result validates the effectiveness of the distributed online algorithm and helps explain its observed success.
    
\end{list}

We perform numerical experiments on movie recommendation, source localization and robot swarm control to corroborate theoretical findings (Section \ref{sec:sims}), and finalize the paper by drawing conclusions (Section \ref{sec:conclusions}).


\section{Wide and Deep Graph Neural Networks} \label{sec:wideDeepGNN}



Let $\ccalG = \{\ccalN, \ccalE, \ccalW\}$ be a graph with a set of $N$ nodes $\ccalN = \{n_{1},\ldots,n_{N}\}$, a set of edges $\ccalE \subseteq \ccalN \times \ccalN$ and an edge-weighing function $\ccalW:\ccalE \to \reals_{\geq 0}$. This graph is used to describe the network topology of the distributed system. For example, in the case of robot swarm coordination, each node $n_{i} \in \ccalN$ models a robot, each edge $(n_{i},n_{j}) \in \ccalE$ models the communication link between robots $n_{i}$ and $n_{j}$, and $\ccalW(n_{i},n_{j}) = w_{ij}$ summarizes the communication channel.

Graph signal processing (GSP) is a mathematical framework for handling distributed problems \cite{Ortega2018-GSP, Gama2019-GLLN, Gama2020-Sketching}. We define a graph signal as a mapping between the nodes and an $F$-dimensional vector space $\mathsf{X}: \ccalN \to \reals^{F}$. The collection of the mappings for all nodes can be conveniently described by a matrix $\bbX \in \reals^{N \times F}$ such that its $i$th row corresponds to $\mathsf{X}(n_{i}) = \bbx_{i} \in \reals^{F}$. We typically refer to each entry of the $F$-dimensional vector (i.e. each column of the matrix $\bbX$ when we consider all nodes) as a feature. The concept of graph signal is used to describe the data associated to the nodes of the graph. For instance, in the robot coordination problem, $\bbx_{i}$ is used to describe the state of robot $n_i$ (acceleration, velocity, relative position),  while $\bbX$ describes the state of the entire system.

Referring to the matrix $\bbX$ as a graph signal is mathematically convenient, but this description loses the connection between the signal and the graph that was originally present in the mapping $\mathsf{X}$. To recover this information, we introduce the concept of graph matrix description (GMD) which is a matrix $\bbS \in \reals^{N \times N}$ that respects the sparsity of the graph, that is $[\bbS]_{ij} = 0$ whenever $(n_{j},n_{i}) \notin \ccalE$ for $i \neq j$. Examples of GMD found in the literature include the adjacency matrix, the Laplacian matrix, and their normalized counterparts \cite{Ortega2018-GSP}. In the context of robot communications, the GMD can be used to represent channel information \cite{Eisen2020-WirelessEGNN}.

To relate the graph signal $\bbX$ to the underlying graph, the GMD $\bbS$ is used to define a linear operation between graph signals such that $\bbY = \bbS \bbX$ with $\bbY \in \reals^{N \times F}$. The $f$th feature value at node $n_{i}$ of the resulting graph signal $[\bbY]_{if}$ is computed as
\begin{equation} \label{eqn:graphShift}
    [\bbS \bbX]_{if} = \sum_{j=1}^{N} [\bbS]_{ij} [\bbX]_{jf} = \sum_{j : n_{j} \in \ccalN_{i} \cup \{n_{i}\}} [\bbS]_{ij} [\bbX]_{jf}
\end{equation}
where $\ccalN_{i} = \{n_{j} \in \ccalN : (n_{j},n_{i}) \in \ccalE\}$ corresponds to the set of nodes that are neighbors of node $n_{i}$. The first equality is the definition of a matrix multiplication, while the second equality holds due to the sparsity pattern of the GMD. Note that it is the sparsity pattern of the GMD that makes the operation in \eqref{eqn:graphShift} a distributed one. That is, in order to compute the output of $\bbS\bbX$ at node $n_{i}$, only local communications with one-hop neighbors $\ccalN_i$ are involved, and the result can be obtained separately at each node. The operation $\bbS \bbX$ is often referred to as a graph shift and the matrix $\bbS$ is therefore called a graph shift operator (GSO). This operation is at the core of GSP because it effectively relates the graph signal with the graph support, and serves as the basic building block for more complex operations between graph signals.

In general, we can think of graph data as given by a pair $(\bbX,\bbS)$ consisting of the graph signal $\bbX$ and its support $\bbS$. We note, however, that we usually regard $\bbX$ as the actionable variable while the support $\bbS$ is determined by the problem setting --this is typically the case for problems involving physical networks. Motivated by \cite{cheng2016wide}, we propose the Wide and Deep Graph Neural Network (WD-GNN) architecture to process graph signals. The WD-GNN is a nonlinear map $\bbPsi:\reals^{N \times F} \to \reals^{N \times G}$ consisting of two components, a wide part $\bbA(\bbX;\bbS,\ccalA)$ and a deep part $\bbPhi(\bbX;\bbS,\ccalB)$, as follows 
\begin{equation} \label{eqn:WDGNN}
    \bbPsi(\bbX;\bbS, \ccalA, \ccalB) = \alpha_{\text{W}} \bbA(\bbX; \bbS, \ccalA) + \alpha_{\text{D}} \bbPhi(\bbX; \bbS, \ccalB) + \beta.
\end{equation}
The wide part $\bbA(\bbX;\bbS,\ccalA)$ is a linear graph filter, that we introduce in Section~\ref{subsec:graphFilter}, while the deep part $\bbPhi(\bbX;\bbS, \ccalA)$ is a nonlinear GNN, that we present in section \ref{subsec:GNN}. The scalars $\alpha_{\text{W}}$, $\alpha_{\text{D}}$, and $\beta$ are combination weights that can either be fixed by the user or learned from data. 

\subsection{Wide part: Linear graph filter} \label{subsec:graphFilter}

A graph filter $\bbA: \reals^{N \times F} \to \reals^{N \times G}$ is a linear mapping between signals that is built upon the graph shift operation \eqref{eqn:graphShift} as follows
\begin{equation} \label{eqn:graphConv}
    \bbA(\bbX; \bbS, \ccalA) = \sum_{k = 0}^{K} \bbS^{k} \bbX \bbA_{k}
\end{equation}
where the set $\ccalA = \{\bbA_{k} \in \reals^{F \times G} \ , \ k = 0,\ldots, K\}$ is the set of $K+1$ filter taps $\bbA_{k}$. Given the GMD $\bbS$, these filter taps completely characterize the graph filter. Note that a graph filter is capable of mapping signals with different feature dimension, i.e., $F$ need not be equal to $G$.

The graph filter is a linear and distributed operation in the input graph signal $\bbX$. To see this, recall that $\bbS \bbX$ is a distributed operation [cf. \eqref{eqn:graphShift}]; therefore, $\bbS^{k}\bbX = \bbS (\bbS^{k-1}\bbX)$ repeats the graph shift $k$ times, which implies that only $k$ exchanges with the one-hop neighbors are required to compute this. The multiplication by the filter tap $\bbA_{k}$ on the left does a linear combination of the feature values at each node separately, so it does not involve any communication between nodes. We note that while the graph filter is compactly written as in \eqref{eqn:graphConv}, in a distributed setting, it is not required to actually compute the powers of a matrix, but simply to be able to communicate with neighbors. As a matter of fact, the nodes need not even know the topology of the graph --see Fig.~\ref{fig.graphcon}. Finally, we observe that the graph filter in \eqref{eqn:graphConv} often receives the name of graph convolution due to its sum-and-shift nature. In the proposed WD-GNN architecture [cf. \eqref{eqn:WDGNN}], we adopt a graph filtering stage as the wide part, making it the linear and distributed component of the architecture.

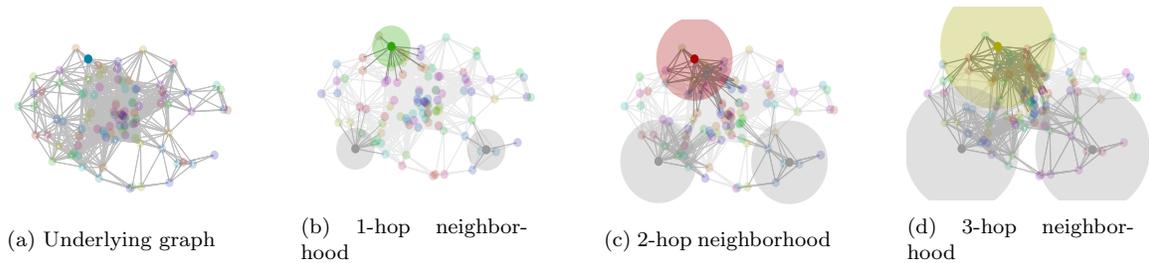
\begin{figure*}[t]
    \begin{subfigure}{0.2\textwidth}
        \includegraphics [width=1.1\linewidth, height = 0.85\linewidth]
        {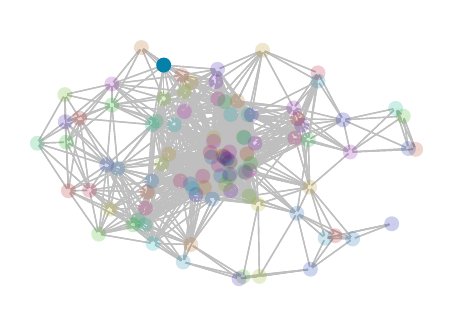}\qquad
        \caption{Underlying graph}%
        \label{subfiga_underlying_graph}%
    \end{subfigure}\hfill\hfill%
    \begin{subfigure}{0.2\textwidth}
        \includegraphics [width=1.1\linewidth, height = 0.85\linewidth]
        {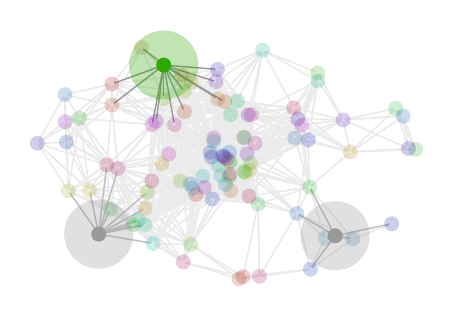}\qquad
        \caption{$1$-hop neighborhood}%
        \label{subfigb_one_hop}%
    \end{subfigure}\hfill\hfill%
    \begin{subfigure}{0.2\textwidth}
        \includegraphics [width=1.1\linewidth, height = 0.85\linewidth]
        {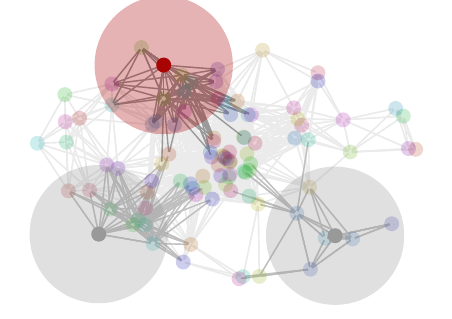}\qquad
        \caption{$2$-hop neighborhood}%
        \label{subfigc_two_hop}%
    \end{subfigure}\hfill\hfill%
    \begin{subfigure}{0.2\textwidth}
        \includegraphics [width=1.1\linewidth, height = 0.85\linewidth]
        {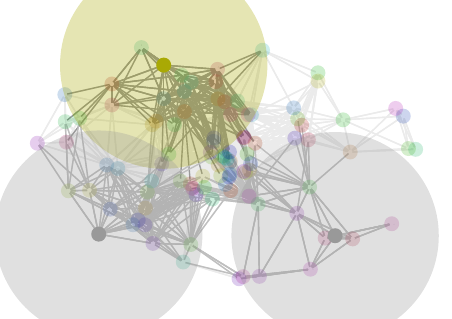} 
        \caption{$3$-hop neighborhood}%
        \label{subfigd_three_hop}%
    \end{subfigure}\\
    \begin{center}
    \tikzstyle {Phi} = [rectangle,
    thin,
    minimum width = 1.0*\unit,
    minimum height = \sumshift*\unit,
    anchor = west,
    draw,
    fill = blue!20]
    
    \tikzstyle {sum} = [circle,
    thin,
    minimum width  = 0.3*\unit,
    minimum height = 0.3*\unit,
    anchor = center,
    draw,
    fill = blue!20]

    \def \deltax {1.5}
    \def \deltay {0.8}
    \def \sumshift {0.4}
    
    \def \thisplotscale {1.75}
    \def \unit {\thisplotscale cm}
    
    \tikzstyle {Phi} = [rectangle,
    thin,
    minimum width = 0.5*\unit,
    minimum height = \sumshift*\unit,
    anchor = west,
    draw,
    fill = blue!20]
    
    \tikzstyle {sum} = [circle,
    thin,
    minimum width  = 0.3*\unit,
    minimum height = 0.3*\unit,
    anchor = center,
    draw,
    fill = blue!20]

    \def \deltax {1.6}
    \def \deltay {0.8}
    \def \sumshift {0.4}
    
    \begin{tikzpicture}[x = 1*\unit, y = 1*\unit]
    
    \node (origin) [] {};
    \path (origin) ++ (0.2*\deltax, 0) node (first) [] {};
    
    \path (first) ++ (1.1*\deltax, 0) node (0) [Phi] {$\bbS$};
    \path (0)     ++ (1.0*\deltax, 0) node (1) [Phi] {$\bbS$};
    \path (1)     ++ (1.0*\deltax, 0) node (2) [Phi] {$\bbS$};
    
    \path (2.east) ++ (0.7*\sumshift*\deltax, 0) node [anchor=west] (last) [] {};
    
    \path (first.east) ++ (1.5*\sumshift*\deltax, -\deltay) node (sum0) [sum] {$+$};
    \path (0.east) ++ (\sumshift*\deltax, -\deltay) node (sum1) [sum] {$+$};
    \path (1.east) ++ (\sumshift*\deltax, -\deltay) node (sum2) [sum] {$+$};
    \path (2.east) ++ (\sumshift*\deltax, -\deltay) node (sum3) [sum] {$+$};
    
    \path[-stealth] (first) edge [very near start, above] node {$\bbX$}               (0);	
    \path[-stealth] (0)     edge [above] node {$\ \bbS\bbX$}     (1);	
    \path[-stealth] (1)     edge [above] node {$\ \bbS^{2}\bbX$} (2);	
    \path[-]        (2)     edge [above] node {$\ \bbS^{3}\bbX$} (sum3 |- last);
    
    \path[-stealth, draw] (sum0 |- first) -- (sum0) node [midway, right] {$\bbA_0$};	
    \path[-stealth, draw] (sum1 |- 0)     -- (sum1) node [midway, right] {$\bbA_1$};	
    \path[-stealth, draw] (sum2 |- 1)     -- (sum2) node [midway, right] {$\bbA_2$};	
    \path[-stealth, draw] (sum3 |- 2)     -- (sum3) node [midway, right] {$\bbA_3$};
    
    \path[-stealth, draw] (sum0) -- (sum1);	
    \path[-stealth, draw] (sum1) -- (sum2);	
    \path[-stealth, draw] (sum2) -- (sum3);	
    
    \path[-stealth] (sum3) edge [above] node
    {$\bbA(\bbX;\bbS,\ccalA)$} ++ (0.7*\deltax, 0);

    \end{tikzpicture}
    \end{center}
    \caption{\small The graph filter performs successive local node exchanges with neighbors, where the $k$-shifted signal $\bbS^k \bbX$ collects the information from $k$-hop neighborhood (shown by the increasing disks), and aggregates these shifted signals $\bbX, \ldots, \bbS^K \bbX$ with a set of parameters $[b^{11}_0, \ldots, b^{11}_K]^\top$ to generate the higher-level feature that accounts for the graph structure up to a neighborhood of radius $K$.\vspace{-0.5cm}}
    \label{fig.graphcon}
\end{figure*}

\subsection{Deep part: Nonlinear graph neural network} \label{subsec:GNN}

A graph convolutional neural network, or GNN for short, is a nonlinear mapping $\bbPhi:\reals^{N \times F} \to \reals^{N \times G}$ between graph signals. GNNs are built as a cascade of blocks (or layers) of graph filters \eqref{eqn:graphConv}, followed by pointwise nonlinearities
\begin{equation} \label{eqn:GCNN}
    \bbPhi(\bbX; \bbS, \ccalB) = \bbX_{L} \quad \text{with} \quad \bbX_{\ell} = \sigma \bigg( \sum_{k=0}^{K_\ell} \bbS^{k} \bbX_{\ell-1} \bbB_{\ell k} \bigg)
\end{equation}
for $\ell = 1,\ldots, L$, where $\sigma:\reals \to \reals$ a pointwise nonlinear function which, for ease of exposition, denotes its entrywise application in \eqref{eqn:GCNN}. The input to the GNN is placed at the first layer $\bbX_{0}=\bbX$ while the output is collected at the last layer $\bbX_{L} = \bbPhi(\bbX;\bbS,\ccalB)$. The GNN is characterized by the set of filter taps $\ccalB = \{ \bbB_{\ell k} \in \reals^{F_{\ell-1} \times F_{\ell}}, k = 0,\ldots, K_\ell, \ell=1,\ldots,L\}$ for each of the filters. Note that each layer maps a signal from an $F_{\ell-1}$-dimensional space into an $F_{\ell}$-dimensional one. The number of layers $L$, the number of output features at each layer $F_{\ell}$ and the number of filter taps at each layer $K_{\ell}$ are all design choices. 

The GNN \eqref{eqn:GCNN} is a nonlinear distributed operation. It is nonlinear because of the effect of the nonlinearity $\sigma$, and it is distributed because it is built on graph filters which are distributed themseleves, and on pointwise nonlinearities that, being pointwise, do not affect the distributed nature of the computation. In the proposed WD-GNN architecture [cf. \eqref{eqn:WDGNN}], we adopt a GNN stage as the deep part making it the nonlinear and distributed component of the architecture. 

\blue{
\begin{remark}\label{remark:distributed}
	The distributed implementation of the graph filtering operation requires multi-hop neighborhood information $\{\bbS^k \bbX\}_{k=1}^K$, which can be obtained by communicating with the neighboring nodes multiple times. Here, the data rate is assumed much slower than the communication rate such that it is sufficient for multi-hop communications in between each generated data. For application scenarios where the communication and data rates are comparable, the graph filtering operation can be adapted to the delayed information structure that preserves the distributed implementation -- see the experiment of robot swarm control in Section \ref{subsec:flocking}.
\end{remark}
}


\section{Online Learning} \label{sec:onlineLearning}



The learning process to train the WD-GNN consists of two phases. First, an offline module, where both the wide and the deep parts are trained using a set of training samples, following standard machine learning procedures --which may be centralized. Second, an online module, where only the wide part is retrained based on the observed testing samples, following a novel, distributed, online optimization algorithm.

\myparagraph{Offline phase.} We train the WD-GNN  \eqref{eqn:WDGNN} by solving the empirical risk minimization (ERM) problem\footnote{We took the license to define the ERM problem as in \eqref{eqn:ERM} so as to include supervised and unsupervised problems in a single framework. To use \eqref{eqn:ERM} for a supervised problem, we just extend $J$ to operate on an extra input representing the label given in the training set.} for some cost function $J: \reals^{N \times G} \to \reals$ over a given training set $\ccalT = \{\bbX_{1},\ldots,\bbX_{|\ccalT|}\}$ as
\begin{equation} \label{eqn:ERM}
    \min_{\ccalA, \ccalB} \frac{1}{|\ccalT|}\sum_{\bbX \in \ccalT} J \big( \bbPsi(\bbX;\bbS, \ccalA, \ccalB) \big).
\end{equation}
The ERM problem on a nonlinear neural network model is typically nonconvex, even if the cost function $J$ is convex. This problem is usually solved by employing some SGD-based optimization algorithm arriving at some set of stationary parameters $\ccalA^{\dag}$ and $\ccalB^{\dag}$ \cite{Vapnik00-StatisticalLearning}. Note that we train the WD-GNN jointly by simultaneously optimizing over the parameters $\ccalA$ of the wide part and the parameters $\ccalB$ of the deep part.

In many practical applications, the problem setting may change from training to testing (transference) or it may change with time (adaptability). For example, in the problem of robot swarm coordination, we have that the initial positions of the robots are different every time (transference) and also that the robots move (adaptability), changing the underlying communication network structure. In such applications, there is a mismatch between the training and testing samples, potentially degrading the performance of the learned model.

\myparagraph{Online phase.} To address this problem, we leverage an online phase of learning, where the wide part of the WD-GNN is retrained by observing testing samples. More concretely, we fix the parameters of the deep part learned during the offline phase, i.e. $\ccalB = \ccalB^{\dag}$, and focus on updating the parameters $\ccalA$ of the wide part. By doing so, we obtain the following optimization problem
\begin{equation}\label{eq:timevaryingp}
    \min_{\ccalA} J_t\big(\bbPsi(\bbX_t; \bbS_t, \ccalA, \ccalB^{\dag})\big)
\end{equation}
where $J_t$, $\bbX_t$ and $\bbS_t$ are the loss function, the observed signal and the graph structure at time $t$, respectively. Fixing the parameters of the nonlinear deep part $\ccalB = \ccalB^{\dag}$ and updating the parameters $\ccalA$ of the linear wide part during this online phase, means that we use the offline module to learn nonlinear representations and the online module to re-learn the linear aspect of this representation.

Retraining only the linear wide part implies that the resulting optimization problem in \eqref{eq:timevaryingp} will be convex if the cost function $J_t$ is convex. Thus, we can leverage algorithms from online learning \cite{dabbagh2005online}. These online algorithms offer convergence guarantees and suboptimality bounds under the assumption that the optimization problem is convex (many useful cost functions like the square loss, the hinge loss, or the logistic loss are convex). In particular, we initialize $\ccalA_0 = \ccalA^{\dag}$ and at time $t$ we have parameters $\ccalA_t$, $\ccalB^{\dag}$, input signal $\bbX_t$, output $\bbPsi(\bbX_t; \bbS_t, \ccalA_t, \ccalB^{\dag})$ and loss $J_t\big(\bbPsi(\bbX_t; \bbS_t, \ccalA_t, \ccalB^{\dag})\big)$. We perform gradient descent with step size $\gamma_t$ to update $\ccalA_t$ as
\begin{equation}\label{eq:onlinelearning}
    \ccalA_{t+1} = \ccalA_{t}-\gamma \nabla_{\ccalA} J_t\big(\bbPsi(\bbX_t; \bbS_t, \ccalA_t, \ccalB^{\dag})\big).
\end{equation}
The above online learning procedure tracks the time-varying convex optimization problem \eqref{eq:timevaryingp} by updating the linear wide part $\ccalA$, while the nonlinear deep part $\ccalB^\dag$ is fixed and learned from the offline phase. Note that online learning algorithms assume a centralized access to the loss function and to the corresponding update. This is a major drawback in decentralized settings where we require algorithms to be implemented distributively. We therefore propose a novel distributed, online learning algorithm.

{
\begin{algorithm}[t] \begin{algorithmic}[1] 
        \small
\STATE \textbf{Input:} offline learned parameters $\ccalA^{\dag}$, $\ccalB^{\dag}$ by minimizing the ERM problem [cf. \eqref{eqn:ERM}] over training set $\ccalT$, online step size $\gamma_t$, and the aggregation weighted matrix $\{\bbW_t\}_t$
\STATE Fix the deep part parameters $\ccalB = \ccalB^{\dag}$ and set the initial wide part parameters $\ccalA_0 = \ccalA^{\dag}$
\FOR {$t = 0,1,2,...$}
      \STATE Observe instantaneous graph signal $\bbX_t$, graph matrix description $\bbS_t$ and loss function $J_t$\\
      \STATE Compute the instantaneous loss $J_t\big(\bbPsi(\bbX_t; \bbS_t, \ccalA_t, \ccalB^{\dag})\big)$\\ 
	  \IF {~requiring decentralized implementation~}
	      \STATE Update the wide part parameters distributively as
	      \FOR {$i = 1,2,...,N$}
          \STATE $\ccalA_{i,t+1} = [\bbW_t]_{ii}\ccalA_{i,t} + \sum_{j:n_j \in \ccalN_i} [\bbW_t]_{ij}\ccalA_{j,t} - \gamma \nabla_{\ccalA_i} J_{i,t}\big(\bbPsi(\bbX_t; \bbS_t, \ccalA_{i,t}, \ccalB^{\dag})\big)$\\
          \ENDFOR
	  \ELSE
	          \STATE Update the wide part parameters as\\
	          \STATE $\ccalA_{t+1} = \ccalA_{t} - \gamma \nabla_\ccalA J_t\big(\bbPsi(\bbX_t; \bbS_t, \ccalA_{t}, \ccalB^{\dag})\big)$\\
      \ENDIF  
\ENDFOR

\end{algorithmic}
\caption{Online Learning Procedure of the WD-GNN}\label{alg:learning}  \label{algo1}
\end{algorithm}}

\myparagraph{Distributed online learning.} In the decentralized setting, each node $n_i$ has access only to a local loss $J_{i,t}\big(\bbPsi(\bbX_t; \bbS_t, \ccalA_i, \ccalB_i)\big)$ with local parameters $\ccalA_i$ and $\ccalB_i$. The goal, therefore, is to coordinate nodes to minimize the mean of local costs $\sum_{i=1}^N J_{i,t}\big(\bbPsi(\bbX_t; \bbS_t, \ccalA_i, \ccalB_i)\big)/N$ while keeping local parameters equal to each other, i.e., $\ccalA_i = \ccalA$ and $\ccalB_i = \ccalB$ for all $i=1,\ldots,N$. This means that we need to recast problem \eqref{eq:timevaryingp} as a constrained optimization problem
\begin{align}\label{eq:distvp}
    &\min_{\{\ccalA_i\}_{i=1}^N} \frac{1}{N}\sum_{i=1}^N J_{i,t}\big(\bbPsi(\bbX_t; \bbS_t, \ccalA_i, \ccalB^{\dag})\big),\\
    &\quad {\rm s.t.}\quad \ccalA_i=\ccalA_j~\forall~i,j\!:\!n_j \in \ccalN_i\nonumber.
\end{align}
Note that $\ccalB_i = \ccalB^{\dag}$ for all $i=1,\ldots,N$ since the deep part is fixed. The constraint $\ccalA_i=\ccalA_j$ for all $i,j:n_j \in \ccalN_i$ indicates that $\ccalA_i=\ccalA$ for all $i=1,\ldots,N$ under the assumed connectivity of the graph.

To solve \eqref{eq:distvp}, at time $t$, each node $n_i$ updates its local parameters $\ccalA_i$ by the recursion
\begin{equation}
\begin{aligned}\label{eq:disgd}
\ccalA_{i,t+1} &= [\bbW_t]_{ii}\ccalA_{i,t} + \sum_{j:n_j \in \ccalN_i} [\bbW_t]_{ij}\ccalA_{j,t} \\
    & \qquad~~~~~ - \gamma \nabla_{\ccalA_i} J_{i,t}\big(\bbPsi(\bbX_t; \bbS_t, \ccalA_{i,t}, \ccalB^{\dag})\big)
\end{aligned}
\end{equation}
\blue{where $\bbW_t \in \mathbb{R}^{N \times N}$ is the weighted matrix accounting for the aggregated weights on local parameters at time $t$. Note that $\bbW_t$ is supported on the graph $\bbS_t$ because each node can only aggregate local parameters of the neighboring nodes via communication in the distributed setting.} Essentially, each node $n_i$ descends its parameters $\ccalA_{i,t}$ along the local gradient $\nabla_{\ccalA_i} J_{i,t}$ $\big(\bbPsi(\bbX_t; \bbS_t, \ccalA_{i,t}, \ccalB^{\dag})\big)$ to approach the optimal solution of \eqref{eq:distvp}, while performing the aggregation over the one-hop neighborhood in the meantime to drive local parameters to the consensus. This online learning algorithm is decentralized and can be carried out locally at each node by only communicating with its one-hop neighbors. \blue{Moreover, there is no underlying assumption about the graph dynamics, i.e., the graph could change randomly across time, such that the distributed online learning is applicable for any time-varying scenarios.}

We summarize the proposed online learning procedure of the WD-GNN in Algorithm \ref{algo1}. We note that it has low complexity due to the linearity, it guarantees efficient convergence due to the convexity (Section~\ref{sec:convergence}), and can be implemented in a distributed manner requiring only neighborhood information. \blue{The combination weights $\alpha_D$ and $\alpha_W$ pursue a trade-off between the nonlinear offline training and the linear online learning. A larger $\alpha_D$ indicates that the deep part, i.e., the nonlinear GNN, plays a more important role and could learn better representations at the offline phase, while a larger $\alpha_W$ indicates that the wide part, i.e., the linear filter, plays a more important role and could adapt faster to the changing scenarios at the online phase.}


\section{Stability Analysis} \label{sec:stability}



In this section, we establish the stability of the WD-GNN to perturbations in the underlying graph $\bbS$. That is, we prove the change of the WD-GNN output caused by the perturbation in $\bbS$ is bounded by the size of the perturbation. This indicates that the WD-GNN is able to maintain certain performance if changes between training and testing are mild.

The WD-GNN consists of a GNN and a graph filter. Both components are permutation equivariant \cite{Gama2020-Stability}, which means that their outputs are unaffected by node reorderings. Since the sum in \eqref{eqn:WDGNN} does not affect this property, the WD-GNN is permutation equivariant. We therefore consider changes in the underlying graph modulo permutation. In particular, given a graph $\bbS$ and its perturbation $\hbS$, we measure the relative perturbation size. Towards this end, define the relative error set as
\begin{equation} \label{eqn:relativeSet}
    \ccalR \!=\! \{\bbE \in \reals^{N \times N} : \bbP^{\Tr} \hbS \bbP \!=\! \bbS + \bbE \bbS + \bbS \bbE \ , \ \bbE \!=\! \bbE^{\Tr}  , \ \bbP \in \ccalP\} \nonumber
\end{equation}
for $\bbP$ in the permutation set $\ccalP =\{\bbP \in \{0,1\}^{N \times N} : \bbP \bbone = \bbone, \bbP^{\Tr} \bbone = \bbone\}$. The relative error set is the set of all symmetric matrices $\bbE$ such that, when multiplied with the support matrix $\bbS$ and added back, they yield a permutation of the perturbed support $\hbS$. Then, we define the \emph{relative error size} between $\bbS$ and $\hbS$ as
\begin{equation} \label{eqn:relativeDistance}
    d(\bbS, \hbS) = \min_{\bbE \in \ccalR} \| \bbE \|.
\end{equation}
The relative error size measures how different the perturbed $\hbS$ is with respect to the original support $\bbS$, irrespective of the specific ordering of the nodes (given the permutation equivariance of the WD-GNN). This relative perturbation model ties the size of the perturbation to the topology of the underlying graph through the multiplication of $\bbE$ with $\bbS$, thus being able to capturing structural information, a feat that is not possible when choosing absolute perturbations $\bbP^{\Tr} \hbS \bbP = \bbS + \bbE$, see \cite{Gama2020-Stability} for details.

The WD-GNN can be proved stable when built with integral Lipschitz filters, which are defined next.
\begin{definition}[Integral Lipschitz filters] \label{def:ILfilters}
    Let $\bbA(\bbX;\bbS,\ccalA)$ be a graph filter [cf. \eqref{eqn:graphConv}]. Let 
    \begin{equation} \label{eq:freqResponse}
        \Big\{ a_{fg}:[\lambda_{l},\lambda_{h}] \to \reals \ , \ a_{fg}(\lambda) \!=\! \sum_{k=0}^{K} [\bbA_{k}]_{fg} \lambda^{k} \Big\}_{f=1,\ldots,F}^{g=1,\ldots,G}
    \end{equation}
    be the frequency response of a graph filter \cite{Ortega2018-GSP}. If there exists a constant $C_L > 0$ such that for all $\lambda_{1},\lambda_{2} \in [\lambda_{l},\lambda_{h}]$ and for all $f=1,\ldots,F$ and $g=1,\ldots,G$, it holds that
    \begin{equation} \label{eqn:ILfilters}
        \big| a_{fg}(\lambda_{2}) - a_{fg}(\lambda_{1})\big| \leq C_L \frac{|\lambda_{2}-\lambda_{1}|}{|\lambda_{1}+\lambda_{2}|/2},
    \end{equation}
    we say that the filter $\bbA(\bbX;\bbS,\ccalA)$ is \emph{integral Lipschitz}.
\end{definition}
\noindent Integral Lipschitz filters are those for which the integral of the filter frequency response is Lipschitz continuous. It is equivalent to the derivative of the filter frequency response satisfying $|\lambda a_{fg}'(\lambda))| \leq C_L$ for all $\lambda$ and for all $f,g$. \blue{This implies that graph filters with finite coefficients $\{\bbA_k\}_{k=1}^K$ and eigenvalues $\{\lambda_i\}_{i=1}^N$ are naturally integral Lipschitz, while the explicit value of $C_L$ depends on the specific use case.} 
Such a condition is also reminiscent of the scale invariance of wavelet transforms. Other examples of integral Lipschitz filters include graph wavelets \cite{Shuman15-Wavelets} and can be enforced by means of penalties during training \cite{Gama2020-Stability}.

Now, we formally establish the stability of the WD-GNN to perturbations of the underlying graph support. Without loss of generality, we assume $F=G=1$ for theoretical analysis.


\begin{theorem}\label{thm:stabilityWDGNN}
    Let $\bbS, \hbS \in \reals^{N \times N}$ be two GMDs such that $d(\bbS,\hbS) \leq \eps$ [cf. \eqref{eqn:relativeDistance}]. Let $\bbPsi(\cdot;\cdot,\ccalA,\ccalB)$ be a WD-GNN \eqref{eqn:WDGNN} with $L$ layers and built on integral Lipschitz filters with constant $C_L$ [cf. \eqref{eqn:ILfilters}] and a nonlinearity $\sigma$ that satisfies $|\sigma(x_1)-\sigma(x_2)| \le |x_1-x_2| $ for all $x_1, x_2 \in \reals$ and $\sigma(0)=0$. Then, it holds that
    \begin{align}\label{eq:H}
        \|\bbPsi(\bbX; \bbS, &\ccalA, \ccalB)-\bbPsi(\bbX; \hbS, \ccalA, \ccalB)\| \\
        &\le 2 C_{L} C_{\Psi} (1+8\sqrt{N}) \| \bbX \| \eps + \ccalO(\eps^2) \nonumber
    \end{align}
    with $C_{\Psi} =|\alpha_D| L \prod\nolimits_{\ell = 1}^{L-1} \!F_\ell + |\alpha_W|$.
\end{theorem}

\begin{proof}
    See Appendix \ref{Appendix:proofOfTheorem1}.
\end{proof}
Theorem 1 states that the WD-GNN is stable to relative graph perturbations. That is, the output difference of the WD-GNN induced by the graph perturbation is proportional to the perturbation size up to a stability constant. This stability constant consists of three terms $C_L$, $C_{\Psi}$ and $1+8\sqrt{N}$. The first constant $C_{L}$ is determined by the filter taps that are learned during training. The second constant $C_{\Psi}$ is impacted by design choices such as the number of layers and the number of features per layer, as well as the mixing between the deep part and the wide part. Finally, the third constant $(1+8\sqrt{N})$ is inherent to the graph topologies under consideration. In any case, we see that the change in the output of a WD-GNN caused due to a change in the graph topology is bounded by the size of that change. \blue{Note that since this stability bound holds uniformly for all graphs, it may not be tight w.r.t. the actual error in some specific graph scenarios. \cite{Gama2020-Stability}.}  


\section{Convergence Analysis} \label{sec:convergence}



We proceed to provide a convergence analysis for the proposed online learning procedure with theoretical performance guarantees. More concretely, we establish convergence to the optimizer of the time-varying problem, up to an error neighborhood that depends on the problem variation.

\subsection{Convergence of centralized online learning}\label{subsec:converCentralized}

We start by considering the centralized online learning procedure [cf. \eqref{eq:onlinelearning}]. Before claiming the main result, we need the following standard assumptions \cite{simonetto2017time}.

\begin{assumption} \label{asm:optimizationTimeVariation}
    Let $J_t\big(\bbPsi(\bbX_t; \bbS_t, \ccalA, \ccalB^\dag)\big)$ be the time-varying loss function of $\ccalA$ with fixed parameters $\ccalB^\dag$ [cf. \eqref{eq:timevaryingp}]. Let also $\ccalA^*_t$ be an optimal solution of $J_t\big(\bbPsi(\bbX_t; \bbS_t, \ccalA, \ccalB^\dag)\big)$ at time $t$. There exists a sequence $\{ \ccalA_t^* \}_t$ and a constant $C_B$ such that for all $t \ge 0$, it holds that 
    \begin{equation}\label{eq:optimizationTimeVariation}
        \| \ccalA_{t+1}^* - \ccalA_t^* \| \le C_B.
    \end{equation}
\end{assumption}

\begin{assumption} \label{asm:optimizationConvexity}
    Let $J_t\big(\bbPsi(\bbX_t; \bbS_t, \ccalA, \ccalB^\dag)\big)$ be the time-varying loss function of $\ccalA$ with fixed parameters $\ccalB^\dag$ [cf. \eqref{eq:timevaryingp}]. For $\bbPsi(\bbX_t; \bbS_t, \ccalA, \ccalB^\dag)$ being a linear function on $\ccalA$, then $J_t\big(\bbPsi(\bbX_t; \bbS_t, \ccalA, \ccalB^\dag)\big)$ is Lipschitz on $\ccalA$ with constant $L$, strongly smooth with constant $C_{t,s}$ and strongly convex with constant $C_{t,c}$. 
\end{assumption}

Assumption \ref{asm:optimizationTimeVariation} establishes the correlation between instantiated problems at successive time indices and bounds the time variation of changing optimal solutions. Assumption \ref{asm:optimizationConvexity} is typical in optimization theory and commonly satisfied in practice \cite{simonetto2017time}. With these assumptions in place, the convergence result is as follows.

\begin{theorem} \label{thm:convergence}
Consider the WD-GNN \eqref{eqn:WDGNN} optimized with the centralized online learning procedure [cf. \eqref{eq:onlinelearning}]. Let $J_t\big(\bbPsi(\bbX_t; \bbS_t, \ccalA, \ccalB^\dag)\big)$ be the time-varying loss function satisfying Assumptions \ref{asm:optimizationTimeVariation}-\ref{asm:optimizationConvexity} with constants $C_B$, $C_{t,s}$ and $C_{t,c}$. Denote by $\ccalA^*_t$ the optimal solution of $J_t\big(\bbPsi(\bbX_t; \bbS_t, \ccalA, \ccalB^\dag)\big)$ and adopt a constant step-size $\gamma_t = \gamma \in (0, 2/C_{t,s})$ for all $t \ge 0$. Then, the sequence $\{ \ccalA_t \}_t$ generated by \eqref{eq:onlinelearning} satisfies
 \begin{equation}\label{eq:centralizedConvergence}
\| \ccalA_{t+1} - \ccalA_{t+1}^* \| \le \big( \prod_{\tau=0}^t m_{\tau} \big) \| \ccalA_0 - \ccalA_0^* \| + \frac{1-\hat{m}^{t+1}}{1-\hat{m}} C_B
\end{equation}
with $m_\tau=\max\{ |1-\gamma C_{\tau,s}|, |1-\gamma C_{\tau,c}| \}$ the convergence rate and $\hat{m}= \max_{0 \le \tau \le t} m_\tau$.
\end{theorem}

\begin{proof}
See Appendix \ref{Appendix:proofOfTheorem2}.
\end{proof}

Theorem \ref{thm:convergence} states that the centralized online learning procedure of the WD-GNN converges to the optimal solution of the time-varying problem, up to a limiting error neighborhood that depends on the change rate of the problem setting. When particularizing $C_B=0$, we have $m_t = m$ for all $t$, yielding the same result as the time-invariant optimization problem, i.e. the linear rate convergence of the gradient descent.

\subsection{Convergence of distributed online learning}

In this section, we focus on the distributed online learning procedure [cf. \eqref{eq:disgd}] and establish its convergence by leveraging the result of its centralized counterpart in Section \ref{subsec:converCentralized}. To do so, we further need assumptions with respect to the graph connectivity and the aggregation weights $\bbW$ [cf. \eqref{eq:disgd}].

\begin{assumption}\label{as:Connectivity}
    Consider the time-varying graph $\ccalG_t(\ccalV, \ccalE_t)$ at time $t$. Let $\ccalE_\infty$ be the set of edges $(i,j)$ whose boundary nodes $n_i$ and $n_j$ are directly connected infinitely many times, i.e., 
    \begin{equation}\label{eq:connectedGraph}
        \ccalE_\infty \!=\! \{ (i,j) | (i,j) \!\in\! \ccalE_t~\text{for infinitely many time indices}~t \}.
    \end{equation}
    The graph $\ccalG_\infty = (\ccalV, \ccalE_\infty)$ is connected.
\end{assumption}

\begin{assumption}\label{as:Interval}
    Consider the graph $\ccalG_\infty = (\ccalV, \ccalE_\infty)$ [cf. \eqref{eq:connectedGraph}]. For any edge $(i,j) \in \ccalE_\infty$, there exists an integer $C_d \ge 1$ such that node $n_i$ is directly connected to node $n_j$ at least every $C_d$ consecutive time indices. 
\end{assumption}

\begin{assumption}\label{as:Doublystochastic}
    The aggregation weights $\bbW_t$ in the distributed online algorithm [cf. \eqref{eq:disgd}] satisfy $[\bbW_t]_{ij} \ge \epsilon$ for any $i,j \in\{ 1,\ldots,N \}$ and is doubly stochastic for all $t \ge 0$, i.e.
    \begin{equation}\label{eq:aggregationWeight}
        \sum_{i} [\bbW_t]_{ij} = \sum_{j} [\bbW_t]_{ij} = 1,~\forall~ t \ge 0.
    \end{equation}
\end{assumption}

Assumption \ref{as:Connectivity} is equivalent to the statement that for any time index $t$, the composite graph $\ccalG = \cup_{\tau \ge t} \ccalG_\tau$ is connected. Assumption \ref{as:Interval} indicates that if the edge $(i,j) \in \ccalE_\infty$ is present in the graph $\ccalG_t$ at time index $t$, it also belongs to one of the graphs in consecutive $C_d$ time indices, i.e.,
\begin{equation}\label{eq:connectedGraph1}
    (i,j) \in \ccalE_{t+1}\cup \ccalE_{t+2}\cup \cdots \cup \ccalE_{t+C_d}.
\end{equation}
\blue{Both assumptions are related to the connectivity of the time-varying graph but do not characterize any specific graph dynamics, which are mild in practice \cite{nedic2009distributed}}. Assumption \ref{as:Doublystochastic} is also easy to satisfy since the weighted matrix $\bbW_t$ is designed by the user (for example, by setting $\bbW_{t}$ to be a normalized, doubly stochastic version of $\bbS_{t}$). 

The aforementioned preliminaries allow us to formally establish the convergence of the distributed online learning procedure in the following theorem.
\begin{theorem} \label{thm:convergenceDistributed}
	Consider the WD-GNN [cf. \eqref{eqn:WDGNN}] optimized with the distributed online learning procedure [cf. \eqref{eq:disgd}]. Let $\{J_{i,t}\big(\bbPsi(\bbX_t; \bbS_t, \ccalA_i, \ccalB^{\dag})\big)\}_{i=1}^N$ be the time-varying local loss function and $\sum_{i=1}^N J_{i,t}\big(\bbPsi(\bbX_t; \bbS_t, \ccalA, \ccalB^{\dag})\big)/N$ be the global loss function satisfying Assumptions \ref{asm:optimizationTimeVariation}-\ref{asm:optimizationConvexity} with constants $C_B$, $L$, $C_{t,s}$ and $C_{t,c}$, $\ccalA^*_t$ be the optimal solution of $\sum_{i=1}^N J_{i,t}\big(\bbPsi(\bbX_t; \bbS_t, \ccalA, \ccalB^{\dag})\big)/N$, and $\gamma_t = \gamma \in (0, 2/C_{t,s})$ be the constant step-size. Let the time-varying graph $\ccalG_t$ satisfy Assumptions \ref{as:Connectivity}-\ref{as:Doublystochastic} with constants $C_d$ and $\epsilon$. Let $\ccalA_{i,0}$ be the initial local parameters satisfying $\max_{1 \le i \le N} \| \ccalA_{i,0} \| \le \gamma L$. Then, the local parameters $\{ \ccalA_{i,t} \}_t$ generated by \eqref{eq:disgd} satisfy
	\begin{align}\label{result:theoremConverDistributed}
		&\|\ccalA_{i,t+1} - \ccalA^*_{t+1}\| \\
		& \le \prod_{\tau=0}^t m_\tau \big\|\frac{1}{N}\sum_{i=1}^N \ccalA_{i,0} \!-\! \ccalA^*_0\big\| \!+\! 2 \gamma C_\epsilon L \Big(\frac{\gamma C_s}{1-\hat{m}} \!+\!1\Big) \!+\!\frac{C_B}{1\!-\!\hat{m}}\nonumber
	\end{align}
	where
	\begin{equation}
		\begin{aligned}\label{constant:theoremConverDistributed}
			C_\epsilon = 1 + \frac{N}{1-(1-\epsilon^{\hat{C_d}})^{1/\hat{C_d}}} \frac{1+\epsilon^{-\hat{C_d}}}{1-\epsilon^{\hat{C_d}}}
		\end{aligned}
	\end{equation}
	is a constant value with $\hat{C_d} = C_d (N-1)$, $m_t=\max\{ |1-\gamma C_{t,s}|, |1-\gamma C_{t,c}| \}$ is the convergence rate, $\hat{m}= \max_{0 \le \tau \le t} m_\tau$ and $C_s= \max_{0 \le \tau \le t} C_{\tau,s}$.
\end{theorem}

\begin{proof}
	See Appendix \ref{Appendix:proofOfTheorem3}.
\end{proof}

Theorem \ref{thm:convergenceDistributed} states that the distributed online learning procedure of the WD-GNN converges to the time-varying optimal solution up to a limiting error neighborhood. The latter depends on the connectivity of the time-varying graph and the optimality variation of the time-varying problem. In particular, the error bound consists of three additive terms: i) the first term $\prod_{\tau=1}^t m_\tau \|1/N\sum_{i=1}^N \ccalA_{i,0} - \ccalA^*_0\|$ decreases to null at a linear rate with the increase of iterations; (ii) the second term $2 \gamma C_\epsilon L (\gamma C_s/(1-\hat{m}) + 1)$ is a constant value proportional to the step-size $\gamma$, which could be sufficiently small by selecting a small step-size $\gamma$; (iii) the third term $C_B/(1-\hat{m})$ is determined by the optimality variation $C_B$ of the time-varying problem, which becomes null for the time-invariant optimization problem. This result characterizes explicitly the converging behavior of the proposed distributed online learning procedure, providing theoretical guarantees for its performance.


\section{Numerical Simulations} \label{sec:sims}



The objective of numerical experiments presented herein, is to evaluate the proposed model on the problems of source localization (Sec.~\ref{subsec:sourceLoc}), robot swarm control (Sec.~\ref{subsec:flocking}) and movie recommendation (Sec.~\ref{subsec:movie}), and corroborate the theoretical findings, i.e., the stability analysis in Theorem \ref{thm:stabilityWDGNN} and the convergence analysis in Theorems \ref{thm:convergence} -\ref{thm:convergenceDistributed} numerically (Sec.~\ref{subsec:stabilityandConvergence}). We also show how the WD-GNN architecture performs in comparison with other commonly used architectures: the GNN and the graph filter.

\blue{
\begin{table*}[t]
	{\small
		\begin{center}
			\caption{ \blue{Performance comparison on source localization.}}
			\label{table_source}
			\begin{tabular}{|l|c|c|}
				\hline
				Architecture$/$Measurement & Training \& testing on the same scenario & Training \& testing on different scenarios \\ \hline
				Graph filter & $41.1618 (\pm 2.0924)$ & $ 25.1600 (\pm 1.3132)$ \\ \hline
				GNN & $93.9673(\pm 1.0987)$ & $ 76.7873 (\pm 2.8290)$ \\ \hline
				WD-GNN &  $97.6745 (\pm 1.1595)$ & $77.4727 (\pm 4.3011)$ \\  \hline
				Graph filter w/ centralized online learning & $38.3897 (\pm 2.0665)$ & $30.2920 (\pm 2.6178)$  \\  \hline
				Graph filter w/ distributed online learning & $39.9552 (\pm 1.5272)$ & $28.0400 (\pm 2.2866)$ \\ \hline
				WD-GNN w/ centralized online learning & $98.0440(\pm 0.6298)$ & $92.7017 (\pm 1.6745)$ \\  \hline
				WD-GNN w/ distributed online learning & $98.1840(\pm 0.4945)$ & $86.5100 (\pm 3.6056)$ \\ 
				\hline
				SGNN \cite{Gao2020-Stochastic} & $93.9673 (\pm 1.0987)$ & $ 84.1133 (\pm 2.7872)$ \\ 
				\hline
			\end{tabular}
		\end{center}
	}
	\vspace{-0.5cm}
\end{table*}
}


\vspace{-0.25cm}

\subsection{Source localization} \label{subsec:sourceLoc}

The goal of this experiment is to localize the source of a diffused signal. Consider a signal diffusion process over a Stochastic Block Model (SBM) graph, which consists of $N=50$ nodes equally divided into $5$ communities with the intra-community link probability $0.8$ and the inter-community link probability $0.2$. There exists a source node $n_{s}$ at each community for $n_s \in \{ n_{s_1}, \ldots, n_{s_5} \} \subset \{n_1,\ldots,n_{N}\}$, and the source signal is initialized as a Kronecker delta $\bbdelta_s = [\delta_1,\ldots, \delta_N]^\top \in \mathbb{R}^N$ with $\delta_s = 1$ at the source node $n_s$ and $\delta_i = 0$ at other nodes $i \ne s$. The diffused signal at time $t$ is given by $\bbx_{ts} = \bbS^t \bbdelta_s + \bbn$ where $\bbS$ is the normalized adjacency matrix and $\bbn \in \mathbb{R}^N$ is an additional Gaussian noise. There exists another detection node $n_d$ at each community for $n_d\in \{n_{d_1}, \ldots, n_{d_5}\} \subset \{n_1,\ldots,n_{N}\}$, at which we determine the source community of a given diffused signal distributively with local neighborhood information.

\myparagraph{Dataset.} We generate the dataset of $13,500$ signal-label samples. The signal is a diffused graph signal $\bbx_{ts}$ with randomly selected source node $n_s$ and diffused time $t$, and the label is the corresponding source community. The dataset is split into $10,000$ samples for training, $2,500$ samples for validation and $1,000$ samples for testing. 

\myparagraph{Parametrization.} We consider a WD-GNN consisting of a graph filter as the wide part and a two-layer GNN as the deep part. Both components have $G=32$ output features, and the deep part has $F_\ell=32$ features per layer. All filters are of order $K=5$ and the nonlinearity is the ReLU. A local readout layer follows to map $32$ output features to a $5$-dimensional vector at each node, indicating predicted probabilities of source communities. We train the WD-GNN for $100$ epochs with batch size of $50$ samples, using the ADAM optimizer \cite{Ba2010} with learning rate $\gamma = 5 \cdot 10^{-3}$ and forgetting factors $0.9$ and $0.999$, respectively. The values of $\alpha_{D}$, $\alpha_{W}$ and $\beta$ are learned from the training data during the offline phase. The performance is measured with the average classification accuracy of all detection nodes, i.e., how many times the source is correctly detected averaged over all detection nodes. Our results are averaged over ten random dataset generations.

\myparagraph{Online learning.} We perform distributed online learning for the graph filter and the WD-GNN at the testing phase. 
In particular, we assume each detection node $n_d$ gets feedback after predicting the source and uses the latter to compute the instantaneous local cost $J_{d, t}\big(\bbPsi(\bbX_t; \bbS_t, \ccalA_{d}, \ccalB^{\dag})\big)$ for $d = d_1,\ldots,d_5$ [cf. \eqref{eq:distvp}]. We consider a distributed online learning procedure experiencing $1,000$ testing signals, which retrains the wide part for each testing signal [cf. \eqref{eq:disgd}]. The combination parameters $\alpha_D$, $\alpha_W$ and $\beta$ are fixed during the online phase. Note that, due to the physical distribution of the nodes, the centralized online learning is not applicable in this scenario, and is therefore only considered as a benchmark.

\myparagraph{Performance.} We perform experiments from two aspects: training and testing on the same problem scenario, and training and testing on different problem scenarios. Table \ref{table_source} summarizes the results. For training and testing on the same scenario, the WD-GNN exhibits the best performance, the GNN follows in the second place, and the graph filter performs worst. This is because the WD-GNN has enhanced representational capability by combining the wide and deep parts while the graph filter is a linear model that cannot extract meaningful features. \blue{The online learning only obtains slight performance improvement for the WD-GNN and even suffers from slight performance degradation for the graph filter since the problem scenario remains the same from training to testing and the offline phase has already trained the models well.}  

For training and testing on different scenarios, we suppose the underlying graph is perturbed by external factors during testing such as channel fading effects, \blue{where edges may be lost with an outage probability $p = 0.3$ -- see details in Section \ref{subsec:stabilityandConvergence}}. The latter results in an edge-dropped subgraph implemented during testing that differs from the underlying graph considered during training. In this case, all architectures suffer from severe performance degradation as observed in Table \ref{table_source}. Both the centralized online learning and the distributed online learning improve performance significantly mitigating this issue. This validates the fact that the online learning successfully adapts the architecture to the perturbed scenario. \blue{The performance improvement of the online learning decreases from the WD-GNN to the graph filter because the latter can only model linear representations at the offline phase and has a limiting performance even with online learning.} It is also worth noting that, while the centralized online learning obtains larger improvements, it is not applicable in practice as it requires global information. \blue{We further compare performance with the SGNN \cite{Gao2020-Stochastic}, which assumes the graph dynamics during testing as the random edge sampling (RES) model and accounts for the latter during training\footnote{\blue{The SGNN reduces to the GNN when training and testing on the same scenario, i.e., the underlying graph without edge dropping.}}. The SGNN exhibits a comparable (slightly worse) performance to the WD-GNN with the distributed online learning. 
However, the RES model cannot represent any graph dynamics and the latter dynamic information may not be available at hand during training, which restrict the application of the SGNN.}
	


\vspace{-0.25cm}

\subsection{Robot swarm control} \label{subsec:flocking}

\begin{table*}[t]
	{\small
		\begin{center}
			\caption{ \blue{Performance comparison on robot swarm control.}}
			\label{table_flocking}
			\begin{tabular}{|l|c|c|}
				\hline
				Architecture$/$Measurement & Total velocity variation & Final velocity variation \\ \hline
				Optimal controller & $52 (\pm 2)$ & $0.0035 (\pm 0.0001)$ \\ \hline
				GNN & $95(\pm 6)$ & $0.0153 (\pm 0.0030)$ \\ \hline
				Graph filter & $428 (\pm 105)$ & $ 1.8 (\pm 1.1)$ \\ \hline
				Graph filter w/ centralized online learning & $388 (\pm 100)$ & $1.3062 (\pm 0.8912)$ \\  \hline
				Graph filter w/ distributed online learning & $400 (\pm 98)$ & $1.4826 (\pm 0.8792)$ \\ 
				\hline
				WD-GNN &  $84 (\pm 5)$ & $0.0119 (\pm 0.0032)$ \\  \hline
				WD-GNN w/ centralized online learning & $79(\pm 4)$ & $0.0065 (\pm 0.0028)$ \\  \hline
				WD-GNN w/ distributed online learning & $82(\pm 4)$ & $0.0069 (\pm 0.0023)$ \\ 
				\hline
			\end{tabular}
		\end{center}
	}
	\vspace{-0.5cm}
\end{table*}

The goal of this experiment is to learn a decentralized controller that coordinates a team of robots to move together at the same velocity while avoiding collisions \cite{Gama2021-ControlGNN}. Consider $N$ robots initially moving at random velocities. At time $t$, each robot $n_i$ is described by its position $\bbp_{i,t} \in \mathbb{R}^2$, velocity $\bbv_{i,t} \in \mathbb{R}^2$ and controls its acceleration $\bbu_{i,t} \in \mathbb{R}^2$ towards the next state
\begin{equation}\label{eq:controlProcess}
	\bbp_{i,t+1} = \bbp_{i,t} + \bbv_{i,t} T_s + \frac{1}{2}\bbu_{i,t}T_s^2, ~ \bbv_{i,t+1} = \bbv_{i,t}+\bbu_{i,t}T_s
\end{equation}
where $T_s$ is the sampling time and $\bbu_{i,t}$ is assumed constant during the sampling time interval $[T_s t,T_s (t+1)]$ --see \cite{Gama2021-ControlGNN} for details. We aim to control accelerations $\bbU_t=[\bbu_{i,t}, \ldots, \bbu_{N,t}]^\top \in \mathbb{R}^{N \times 2}$ such that robots with random initial velocities will ultimately move at the same velocity without collision. This problem has a centralized controller that can be readily computed as \cite{Gama2021-ControlGNN}
\begin{equation}\label{eq:optimalController}
	\bbu_{i,t}^* = - \sum_{j=1}^N \left(\bbv_{i,t} - \bbv_{j,t} \right) - \sum_{j=1}^N \nabla_{\bbp_{i,t}} V\left( \bbp_{i,t}, \bbp_{j,t} \right)
\end{equation}
for all $i=1,\ldots,N$ with $V( \bbp_{i,t}, \bbp_{j,t})$ the collision avoidance potential. However, the computation of such a solution $\bbu_{i,t}^*$ requires the knowledge of positions $\{ \bbp_{i,t} \}_{i=1}^N$ and velocities $\{ \bbv_{i,t} \}_{i=1}^N$ of all robots over network and thus demands a centralized computation unit. The latter may not be available in practice, especially for large-scale networks \cite{Gama2021-DistributedLQR}.

In the decentralized setting, each robot only has access to local neighborhood information. We assume robot $n_i$ can communicate with robot $n_j$ and obtain its information if and only if they are within the communication radius $r$, i.e., there is a communication link $(i,j)$ if $\| \bbp_{i,t}-\bbp_{j,t} \| \le r$ at time $t$. We establish the communication graph $\ccalG_t$ with the node set $\ccalV = \{ n_1,\ldots,n_N \}$ and the edge set $\ccalE_t$ containing available communication links. The support matrix $\bbS_t$ is the adjacency matrix with entry $[\bbS_t]_{ij} =1$ if $(i,j) \in \ccalE_t$ and $[\bbS_t]_{ij} =0$ otherwise. We apply the WD-GNN to learn a decentralized controller $ \bbU_t = \bbPsi(\bbX_t; \bbS_t, \ccalA, \ccalB)$ where the graph signal $\bbX_t\!=\![ \bbx_{1,t}, \ldots, \bbx_{N,t}]^\top \!\in\! \mathbb{R}^{N \times 6}$ is
\begin{equation}\label{eq:relevantFeature}
	\begin{aligned}
		\bbx_{i,t} \!= &\Big[\!\sum_{j:n_j\in \ccalN_{i,t}}\!\!\!\!\!\!\big(\bbv_{i,t} \!-\! \bbv_{j,t} \big), \sum_{j:n_j\in \ccalN_{i,t}}\frac{\bbp_{i,t}-\bbp_{j,t}}{\| \bbp_{i,t}-\bbp_{j,t} \|^4}, \\
		& \quad \quad \quad \quad \quad \quad \quad \quad \quad  \sum_{j:n_j\in \ccalN_{i,t}}\frac{\bbp_{i,t}-\bbp_{j,t}}{\| \bbp_{i,t}-\bbp_{j,t} \|^2}\Big]
	\end{aligned}
\end{equation}
for all $i=1,\ldots,N$, which is a local feature vector collecting relative position and velocity information of neighboring robots \cite{Gama2021-ControlGNN}. \blue{In this case, both the robot states $\bbX_t$ and the communication graph $\bbS_t$ change at each time $t$, i.e., the data rate coincides the communication rate. Multi-hop communications may not be applicable w.r.t. each $\bbX_t$ for the distributed implementation -- see Remark \ref{remark:distributed}, and 
the graph filter in \eqref{eqn:graphConv} is adapted to the delayed information structure $\{\bbX_t, \bbS_t\}_t$ as 
\begin{equation} \label{eqn:delayedGraphFilter}
	\bbA(\bbX_t; \bbS_t, \ccalA) = \sum_{k = 0}^{K} \bbS_t \bbS_{t-1} \cdots \bbS_{t-(k-1)} \bbX_{t-k} \bbA_{k}.
\end{equation}
}We leverage the imitation learning framework \cite{Ross10-ImitationLearning} to train the WD-GNN distributed controller by imitating the expert centralized controller \eqref{eq:optimalController}.

\myparagraph{Dataset.} The dataset contains $400$ trajectories for training, $40$ for validation and $40$ for testing. We generate each trajectory by initially positioning $N=50$ robots randomly in a circle. The minimal initial distance between two robots is $0.1$m and initial velocities are sampled randomly from $[-v, +v]^2$ with $v=3$m/s. The duration of trajectories is $T=2$s with the sampling time $T_s=0.01$s, the maximum acceleration is $\pm 10 \text{m}/{\text{s}^2}$, and the communication radius is $r=2$m.

\myparagraph{Parametrization.} For the WD-GNN, we consider the wide part as a graph filter and the deep part as a single-layer GNN, where both have $G=32$ output features. All filters are of order $K=3$ and the nonlinearity is the Tanh. The output features are fed into a local readout layer to generate two-dimensional acceleration $\bbu_{i,t}$ at each robot $n_i$. We train the WD-GNN for $30$ epochs with batch size of $20$ samples, using the ADAM optimizer \cite{Ba2010} with learning rate $\gamma = 5 \cdot 10^{-4}$ and forgetting factors $0.9$ and $0.999$, respectively. The values of $\alpha_{D}$, $\alpha_{W}$ and $\beta$ are learned from the training data during the offline phase. We average experimental results for $5$ dataset realizations. 

\myparagraph{Measurements.} The flocking condition is quantified by the variance of robot velocities, referred to as \emph{velocity variation}. Specifically, we measure the performance of the learned controller from two aspects: the total velocity variation over the whole trajectory $\sum_{t=1}^D \sum_{i=1}^N \big\| \bbv_{i,t} - \sum_{i=1}^N \bbv_{i,t}/N \big\|^2/N$ and the final velocity variation $\sum_{i=1}^N \big\| \bbv_{i,D} - \sum_{i=1}^N \bbv_{i,D}/N \big\|^2/N$ at the last time index $D=T/T_s$. The former reflects the whole controlling process which decreases if robots approach the consensus more quickly, while the latter tells how well the final flocking condition is achieved.

\myparagraph{Online learning.} We perform both the centralized and the distributed online learning for the graph filter and the WD-GNN during testing. The former uses the velocity variation over all robots as the instantaneous cost in \eqref{eq:timevaryingp}, which requires velocities of all robots and is not practical. The latter uses the velocity variance over neighboring robots as the instantaneous local cost in \eqref{eq:distvp}, which leverages neighborhood information and can be implemented distributively. The combination parameters $\alpha_D$, $\alpha_W$ and $\beta$ are fixed during the online phase.
%
\begin{table*}[t]\small
	\caption{\blue{Performance of the WD-GNN with different filter orders $K$ and different numbers of features $F$.}} 
	\centering 
	\begin{tabular}{|c | c| p{1cm} p{1cm} p{1cm}|p{2.2cm} p{2.2cm} p{2.2cm}|} 
		\hline 
		& Measurement &  \multicolumn{3}{|c|}{Total velocity variation} &  \multicolumn{3}{|c|}{Final velocity variation}   \\
		Architecture &  & ~F=16 & ~F=32 & ~F=48& ~~~~~~F=16 & ~~~~~~F=32 & ~~~~~~F=48
		\\
		\hline 
		& WD-GNN &$115 (\pm 9)$ & $90 (\pm 8)$ & $85 (\pm 5)$& $0.0549(\pm 0.0233)$ & $0.0206 (\pm 0.0090)$ & $0.0132 (\pm 0.0031)$  \\
		K=2 & Centralized online learning
		& $108 (\pm 9)$ & $86 (\pm 7)$ & $80 (\pm 5)$ & $0.0270 (\pm 0.0193)$ & $0.0063 (\pm 0.0030)$ & $0.0036 (\pm 0.0022)$ \\& Distributed online learning & $111 (\pm 8)$ & $88 (\pm 8)$ & $83 (\pm 5)$ & $0.0329 (\pm 0.0177)$ & $0.0121 (\pm 0.0065)$ & $0.0088 (\pm 0.0026)$  \\ \hline
		& WD-GNN & $112 (\pm 11)$ & $84 (\pm 5)$ & $86 (\pm 3)$& $0.0395 (\pm 0.0251)$ & $0.0119 (\pm 0.0032)$ &$0.0109 (\pm 0.0049)$  \\
		K=3 & Centralized online learning
		& $107 (\pm 9)$ & $79 (\pm 4)$ & $81 (\pm 3)$ & $0.0229 (\pm 0.0179)$ & $0.0065 (\pm 0.0028)$ & $0.0059 (\pm 0.0025)$  \\& Distributed online learning
		& $109 (\pm 10)$ & $82 (\pm 4)$ & $84 (\pm 3)$& $0.0313 (\pm 0.0225)$ & $0.0069 (\pm 0.0023)$ & $0.0077 (\pm 0.0041)$  \\ \hline
		& WD-GNN &$116 (\pm 10)$ & $88 (\pm 6)$ & $83 (\pm 3)$& $0.0494 (\pm 0.0156)$ & $0.0132 (\pm 0.0035)$ & $0.0108 (\pm 0.0028)$  \\
		K=4 & Centralized online learning
		& $107 (\pm 10)$ & $83 (\pm 6)$ & $79 (\pm 2)$& $0.0176 (\pm 0.0118)$ & $0.0058 (\pm 0.0022)$ & $0.0027 (\pm 0.0017)$  \\& Distributed online learning
		& $112 (\pm 9)$ & $85 (\pm 6)$ & $81 (\pm 2)$& $0.0299 (\pm 0.0116)$ & $0.0065 (\pm 0.0023)$ & $0.0056 (\pm 0.0021)$  \\
		\hline 
	\end{tabular}
	\label{tab:ComparisonDropEdge}
	\vspace{-0.5cm}
\end{table*}

\myparagraph{Performance.} Since the robot swarm system initializes robot conditions randomly and evolves dynamically across time, problem scenarios are typically different between training and testing. Besides the GNN and the graph filter, we also compare the optimal controller [cf. \ref{eq:optimalController}] for reference. Table \ref{table_flocking} shows the results. We see that the WD-GNN exhibits the best performance in both performance measures (except the optimal centralized controller). This can be explained by the increased representational capacity of the WD-GNN as a combined architecture. The GNN takes the second place, while the graph filter performs much worse since the optimal controller is known to be nonlinear \cite{Witsenhausen68-Counterexample}. The online learning procedure reduces both the total and final velocity variations, which implies that it successfully adapts the architectures to the changing initial conditions and communication graphs. The reduction in the final velocity variation is more noticeable, since the effect of single-time online gradient updates [cf. \eqref{eq:disgd}] gets compounded as the trajectory proceeds. \blue{Moreover, while the online learning of the graph filter obtains larger variation reductions, its performance is too bad for consideration because the graph filter can only learn linear relations between the robot states and the action policy at the offline phase.

In addition, we evaluate the WD-GNN with the online learning under different architecture hyper-parameters, i.e., different numbers of features $F$ and different filter orders $K$. The online learning reduces the total and final velocity variations for all architectures, where the centralized one outperforms but requires global information. In terms of the hyper-parameter effects, the expected performance improves and the standard deviation decreases with the number of features $F$ but not necessarily the filter order $K$. 
While the optimal controller outperforms the WD-GNN, the latter only needs local neighborhood information for implementation and the final velocity variation of the centralized online learning reduces to a comparable value to the optimal controller in certain hyper-parameter settings. 
}


\subsection{Movie recommendation} \label{subsec:movie}

The goal of this experiment is to predict the rating a user would give to a specific movie, based on the ratings that user -- and other users -- have given to some other collection of movies \cite{Huang2018}. We build the underlying graph as the movie similarity network, where nodes are movies and edge weights are similarity ratings between movies. The graph signal contains the ratings of movies given by a user, with missing values if those movies are not rated by that user. We train the WD-GNN to predict the rating of a movie of our choice, based on ratings given to other movies.

\myparagraph{Dataset.} We consider a subset of MovieLens-100k dataset \cite{Harper2016}, containing $943$ users and $400$ movies with largest number of ratings. We compute the movie similarity as the Pearson correlation and keep the ten edges with highest similarity for each node (movie) --see \cite{Huang2018} for details. The ratings given by each user are modeled as a graph signal, where the value on each node is the rating given to the corresponding movie, or zero if that movie has not been rated. The dataset is split into $90 \%$ for training and $10\%$ for testing. The objective is to estimate the rating of the movie `\emph{Star Wars}' because it is the one with the largest number of ratings given by the users.

\myparagraph{Parametrization.} We consider a WD-GNN whose wide part is a graph filter and deep part is a single-layer GNN. Both components have $G=64$ output features, filters are of order $K=5$, and the nonlinearity is the ReLU. A local readout layer follows to map $64$ output features to a scalar predicted rating at each node. We train the WD-GNN for $30$ epochs with batch size of $5$ samples, using the ADAM optimizer \cite{Ba2010} with learning rate $\gamma = 5 \cdot 10^{-3}$ and forgetting factors $0.9$ and $0.999$, respectively. The values of $\alpha_{D}$, $\alpha_{W}$ and $\beta$ are learned from the training data during the offline phase. The performance is measured with the root mean squared error (RMSE), averaged over $10$ random dataset splits. The estimated standard deviation is also shown.

\myparagraph{Online learning} We perform online learning for the graph filter and the WD-GNN during testing. Similarly as source localization, we assume the recommendation system gets feedback from the user after it predicted the rating. This feedback is used as the label to compute the instantaneous loss function $J_t\big(\bbPsi(\bbX_t; \bbS_t, \ccalA, \ccalB^{\dag})\big)$ [cf. \eqref{eq:timevaryingp}]. We consider an online learning procedure experiencing $400$ testing users and for each user, the system performs gradient descent [cf. \eqref{eq:onlinelearning}] to retrain the wide part based on the instantaneous cost. The combination parameters $\alpha_D$, $\alpha_W$ and $\beta$ are fixed during the online phase. While, in practice, centralized online learning is applicable for recommendation systems, we keep in mind that we are concerned with distributed online execution of the WD-GNN --see Sections \ref{subsec:sourceLoc} and \ref{subsec:flocking} for other examples where this assumption has a physical justification.

\myparagraph{Performance.} We similarly perform experiments from two aspects: training and testing on the same movie, and training on one movie and testing on a different one. The second experiment exemplifies a case where the problem scenario significantly changes from training to testing. Table \ref{table_recommendation} shows the results. We see that three architectures exhibit comparable performance, while the WD-GNN performs best with the lowest RMSE. We attribute this behavior to the enhanced learning ability of the WD-GNN as a combined architecture. The online learning only improves performance slightly 
because the problem scenario does not change and the offline phase of the WD-GNN seems to have already captured the main challenges of the problem.

For training on one movie and testing on a different one, we choose to train on \emph{Star Wars}, but test on \emph{Contact} and \blue{\emph{Return of Jedi}}. The results are also summarized in Table \ref{table_recommendation}. As expected, all three architectures experience performance degradation due to the change of the problem scenario. However, leveraging online learning improves the performance significantly as it successfully adapts the architectures to the new scenario. \blue{The online learning of the graph filter achieves a comparable performance to that of the WD-GNN in this experiment because the linear representations of the graph filter solve the problem as well as the nonlinear ones of the WD-GNN, but it is worth remarking that this is not necessarily the case for other applications -- see Sections \ref{subsec:sourceLoc} - \ref{subsec:flocking}.
}

\begin{table*}[t]
	{\small
		\begin{center}
			\caption{\blue{Performance comparison on movie recommendation.}}
			\label{table_recommendation}
			\begin{tabular}{|l|c|p{2.5cm} p{2.5cm}|}
				\hline
				 & Train $\&$ test on same movie &  \multicolumn{2}{|c|}{Train on one movie $\&$ test on another} \\
				Architecture$/$Experiments & Star War & ~~~~~~Contact  & ~~Return of Jedi  \\ \hline
				GNN & $0.8630(\pm 0.0884)$ & $1.0889 (\pm 0.1106)$ & $1.0072 (\pm 0.1039)$\\ \hline
				Graph filter & $0.8589 (\pm 0.0895)$ & $ 1.0950 (\pm 0.1129)$ & $0.9924 (\pm 0.1005)$\\ \hline
				WD-GNN &  $0.8535 (\pm 0.0883)$ & $1.0920 (\pm 0.1053)$ & $0.9986 (\pm 0.0930)$\\ \hline
				Graph filter w/ online learning & $0.8595 (\pm 0.0865)$  & $0.9837 (\pm 0.0898)$ &  $0.9486 (\pm 0.0755)$\\  \hline
				WD-GNN w/ online learning & $0.8524(\pm 0.0907)$ & $0.9759 (\pm 0.0927)$ & $0.9232 (\pm 0.0664)$ \\  \hline
			\end{tabular}
		\end{center} 
	}
	\vspace{-0.5cm}
\end{table*}

\subsection{Stability and Convergence Corroboration}\label{subsec:stabilityandConvergence}

The goal of this experiment is to corroborate the stability analysis in Theorem \ref{thm:stabilityWDGNN} and the convergence analysis in Theorems \ref{thm:convergence}-\ref{thm:convergenceDistributed}. We consider the problem of source localization and follow the experimental setting in Section \ref{subsec:sourceLoc}. 

\myparagraph{Stability corroboration.} We consider the WD-GNN is trained on the underlying graph $\bbS$ but implemented on the perturbed graph $\hat{\bbS}$ during testing. \blue{In particular, we suppose all edges of the underlying graph may fall with a probability $p$, due to external factors such as channel fading, object blocking and human effects. The latter results in an edge-dropped subgraph during testing that differs from the underlying graph during training.} Fig. \ref{fig:stability} shows the performance difference of the WD-GNN induced by the graph perturbation, where different edge dropping probabilities $p$ imply different severity of the perturbation. We see little performance degradation when $p$ is close to zero, which indicates the WD-GNN maintains performance when the number of dropping edges is small and the graph perturbation is mild. The classification accuracy decreases slightly as $p$ becomes larger and more edges get dropped. The results corroborate the stability analysis present in Theorem \ref{thm:stabilityWDGNN}.

\myparagraph{Convergence corroboration.} We consider the WD-GNN is trained on the underlying graph and tested on the edge-dropped subgraph with the edge dropping probability $p=0.3$. In this scenario, the perturbation is significant and the online learning is leveraged to mitigate the performance degradation. The online learning procedure experiences $1000$ testing signals, each of which retrains the wide part with gradient descent. Fig. \ref{fig:convergence} shows the convergence process. In particular, the classification accuracy increases as the testing phase proceeds, i.e., as the number of testing signals increases, leading to a convergent result in both centralized and distributed cases. The centralized online learning converges fast but is not practical, while the distributed online learning achieves comparable performance and can be implemented in the decentralized setting. The results corroborate the convergence analysis present in Theorems \ref{thm:convergence}-\ref{thm:convergenceDistributed}.

\begin{figure}[t]
	\centering
	\begin{subfigure}{.24\textwidth}
		\includegraphics[width=\textwidth]{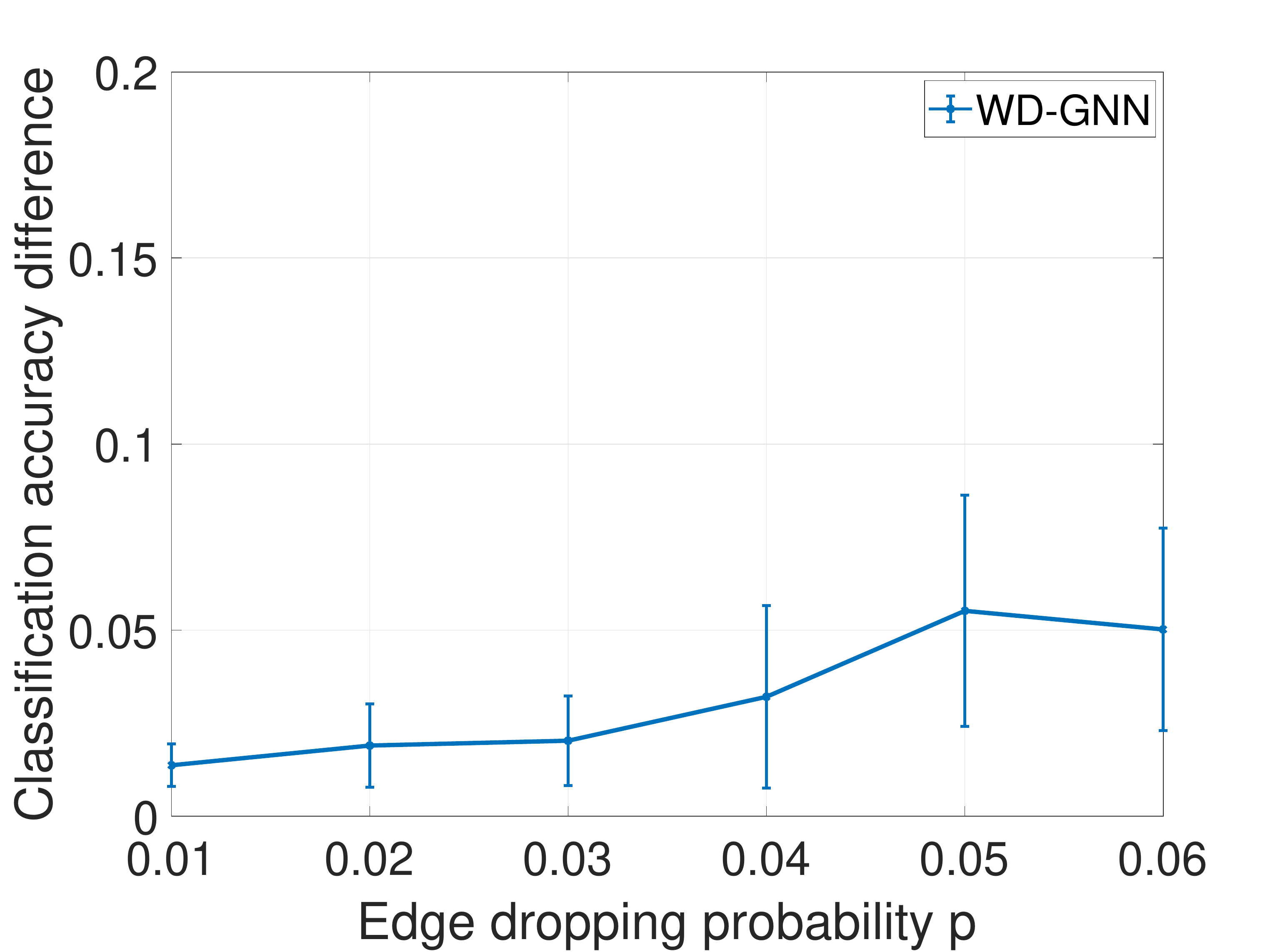}
		\caption{}
		\label{fig:stability}
	\end{subfigure}
	\begin{subfigure}{.24\textwidth}
		\includegraphics[width=\textwidth]{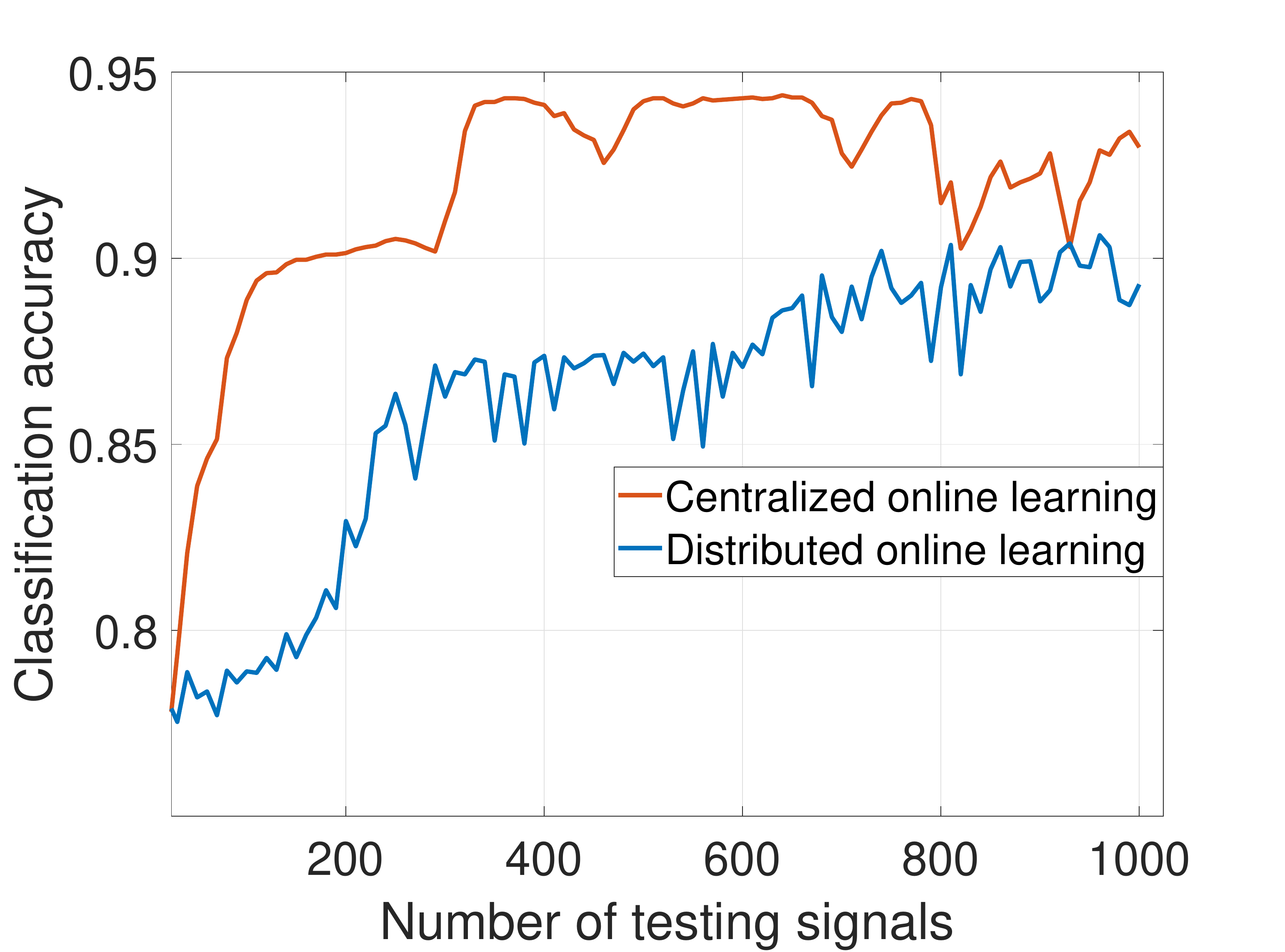}
		\caption{}
		\label{fig:convergence}
	\end{subfigure}
	\caption{\small (a) Performance difference under different edge dropping probabilities $p$. (b) Convergence process of online learning.\vspace{-0.5cm}}
	\label{fig:stabilityandConvergence}
\end{figure}  


\section{Conclusions} \label{sec:conclusions}



This paper proposed the Wide and Deep Graph Neural Network as a joint architecture comprising of a linear graph filter (wide part) and a nonlinear GNN (deep part). To alleviate performance degradation induced by the change of the problem scenario between training and testing, the learning process of the WD-GNN consists of two phases. The offline phase trains the wide and deep parts jointly to learn nonlinear representations from the training data. The online phase fixes the deep part and only retrains the wide part adapting to the testing data, which results in solving a (typically convex) time-varying optimization problem. The latter allows for efficient online learning with provable convergence guarantees. We further developed a distributed online learning algorithm, which can be implemented in a decentralized setting and corresponds with the distributed nature of the WD-GNN. The stability analysis shows the architecture robustness to mild graph perturbations, and the convergence analysis evaluates the efficiency of the distributed online learning algorithm. Numerical experiments are performed on movie recommendation, source localization and robot swarm control, corroborating the effectiveness of the proposed WD-GNN with distributed online learning. \blue{In near future, we plan to corroborate these findings with other datasets and in other applications, such as the MovieLens1M dataset for movie recommendation and the motion planning problem in multi-agent systems.}


\ifundefined{arXiv}
    \appendices
\else
    \appendix
\fi


\section{Proof of Theorem 1} \label{Appendix:proofOfTheorem1}

\begin{proof}
The output difference between $\bbPsi(\bbX; \bbS, \ccalA, \ccalB)$ and $\bbPsi(\bbX; \hbS, \ccalA, \ccalB)$ can be divided into the wide part difference and the deep part difference as
 \begin{align}\label{eqn:thm100}
&\|\bbPsi(\bbX;\! \bbS, \ccalA,\! \ccalB)\!-\!\bbPsi(\bbX; \!\hbS, \ccalA,\! \ccalB)\| \\
&\le\! \|\bbA(\bbX;\! \bbS,\! \ccalA) \!-\! \bbA(\bbX;\! \hat{\bbS},\! \ccalA)\| \!+\! \|\bbPhi(\bbX;\! \bbS,\! \ccalB)\!-\!\bbPhi(\bbX;\! \hbS,\! \ccalB)\| \nonumber
\end{align}
where the triangle inequality is used. Let us consider these two terms separately.

\textbf{The wide part difference.} By using Lemma \ref{lemma1}, we bound the wide part difference as 
 \begin{align} \label{eqn:thm11}
    \| \bbA(\bbX; \bbS,\! \ccalA) \!-\! \bbA(\bbX; \hat{\bbS},\! \ccalA)\| &\!\le\! 2C_L(1 \!+\! \sigma \sqrt{N}) \eps \!+\! \ccalO(\eps^2) \\
    &\le 2C_L(1 + 8 \sqrt{N}) \eps + \ccalO(\eps^2) \nonumber
\end{align}
where $\sigma \le 8$ is used in the last inequality since $\|\bbU\| = \|\bbV\|=1$ in \eqref{eq:lemma11}.  
 
\textbf{The deep part difference.} At layer $\ell$, the graph convolution [cf. \eqref{eqn:graphConv}] can be considered as the application of $F_{\ell-1}F_\ell$ one-dimensional input and one-dimensional output filters, i.e.,
\begin{equation}\label{eqn:thm12} 
 \left[\sum_{k=0}^{K} \bbS^{k} \bbX_{\ell-1} \bbB_{\ell k}\right]_f = \sum_{g=1}^{F_{\ell-1}} \sum_{k=0}^K [\bbB_{\ell k}]_{fg} \bbS^k [\bbX_{\ell-1}]_{g} 
\end{equation}
for $f=1,\ldots, F_\ell$, where $[\bbB_{\ell k}]_{fg}$ is the $(f,g)$th entry of matrix $\bbB_{\ell k}$ and $[\cdot]_f$ represents the $f$th column. We denote by $\bbG_\ell^{fg}(\bbS)[\bbX_{\ell-1}]_{g}= \sum_{k=0}^K [\bbB_{\ell k}]_{fg} \bbS^k [\bbX_{\ell-1}]_{g}$ and $\bbx^g_{\ell-1} = [\bbX_{\ell-1}]_{g}$ for convenience of following derivations. By using \eqref{eqn:thm12} and substituting the GNN architecture [cf. \eqref{eqn:GCNN}] into the deep part difference, we get
 \begin{align} \label{eqn:thm13}
&\|\bbPhi(\bbX;\! \bbS,\! \ccalB)\!-\!\bbPhi(\bbX;\! \hbS,\! \ccalB)\| \\
&= \Big\| \sigma\Big( \sum_{f=1}^{F_{L-1}} \bbG_L^{1f}(\bbS)\bbx^f_{L-1} \Big) - \sigma\Big( \sum_{f=1}^{F_{L-1}} \bbG_L^{1f}(\hbS)\hat{\bbx}^f_{L-1} \Big) \Big\| \nonumber\\
&\le \sum_{f=1}^{F_{L-1}}  \big\| \bbG_L^{1f}(\bbS)\bbx^f_{L-1} - \bbG_L^{1f}(\hbS)\hat{\bbx}^f_{L-1} \big\| \nonumber
\end{align}
where $\hat{\bbx}_{L-1}^f$ is the $f$th feature of the $(L-1)$th layer output when acting on $\hat{\bbS}$ instead of $\bbS$, and where the second inequality is because of the normalized Lipschitz nonlinearity $\sigma$ and the triangle inequality. By adding and substracting $\bbG_L^{1f}(\hbS)\bbx^f_{L-1}$, the term in the norm is bounded by
 \begin{align} \label{eqn:thm14}
&\big\| \bbG_L^{1f}\!(\bbS)\bbx^f_{L\!-\!1} \!\!-\! \bbG_L^{1f}\!(\hbS)\bbx^f_{L\!-\!1} \!\big\| \!+\!\! \big\|\bbG_L^{1f}\!(\hbS)\bbx^f_{L\!-\!1} \!\!-\! \bbG_L^{1f}\!(\hbS)\hat{\bbx}^f_{L\!-\!1} \!\big\| \nonumber \\
& \!\!\le \!\!\big\| \bbG_L^{1f}\!(\bbS)\bbx^f_{L\!-\!1} \!\!-\!\! \bbG_L^{1f}\!(\hbS)\bbx^f_{L\!-\!1} \!\big\| \!\! +\!\! \| \bbG_L^{1f}\!(\hbS) \|\! \big\|\bbx^f_{L\!-\!1} \!\!\!-\! \hat{\bbx}^f_{L\!-\!1}\! \big\|.
\end{align}
For the first term in \eqref{eqn:thm14}, by using Lemma \ref{lemma1} and \eqref{eqn:thm11}, we get
\begin{align} \label{eqn:thm15}
&\big\| \bbG_L^{1f}\!(\bbS)\bbx^f_{L-1} - \bbG_L^{1f}(\hbS)\bbx^f_{L-1} \big\| \\
& \le 2C_L(1 + 8 \sqrt{N}) \|\bbx^f_{L-1}\| \eps + \ccalO(\eps^2).\nonumber
\end{align}
In terms of the term $\|\bbx^f_{L\!-\!1}\|$, we observe that 
\begin{align} \label{eqn:thm16}
&\|\bbx^f_{L\!-\!1}\| = \Big\| \sigma\Big( \sum_{g=1}^{F_{L-2}} \bbG_{L-1}^{fg}(\bbS)\bbx^g_{L-2} \Big) \Big\| \\
&\le  \sum_{g=1}^{F_{L-2}} \big\| \bbG_{L-1}^{fg}(\bbS)\bbx^g_{L-2} \big\| \le \sum_{g=1}^{F_{L-2}} \big\| \bbx^g_{L-2} \big\|\nonumber
\end{align}
where we use the triangle inequality, followed by the bound on filters, i.e., the filter frequency response $|b^{fg}_{L-1}(\lambda)| = \big| \sum_{k=0}^K [\bbB_{(L-1) k}]_{fg} \lambda^k \big| \le 1$. We follow this recursion to obtain $\|\bbx^f_{L\!-\!1}\| \le \prod_{\ell = 1}^{L-2} F_\ell \big\| \bbx^1_0 \big\|$, 
where $\| \bbx^1_0 \| = \| \bbX \|$ by definition since the dimension of the input feature is $F_0=F=1$. By substituting this result into \eqref{eqn:thm15}, we have
\begin{align} \label{eqn:thm18}
&\big\| \bbG_L^{1f}\!(\bbS)\bbx^f_{L-1} - \bbG_L^{1f}\!(\hbS)\bbx^f_{L-1} \big\| \\
&\le  2C_L(1 + 8 \sqrt{N})\prod_{\ell = 1}^{L-2} F_\ell \big\| \bbX \big\| \eps + \ccalO(\eps^2). \nonumber
\end{align}
For the second term in \eqref{eqn:thm14}, by again using the filter bound, we get
\begin{equation} \label{eqn:thm19}
\| \bbG_L^{1f}\!(\hbS) \| \big\|\bbx^f_{L\!-\!1} \!-\! \hat{\bbx}^f_{L\!-\!1} \big\| \le \big\|\bbx^f_{L\!-\!1} \!-\! \hat{\bbx}^f_{L\!-\!1} \big\|.
\end{equation}
By substituting \eqref{eqn:thm18} and \eqref{eqn:thm19} into \eqref{eqn:thm14} and the latter into \eqref{eqn:thm13}, we have
 \begin{align} \label{eqn:thm110}
&\|\bbPhi(\bbX;\! \bbS,\! \ccalB)\!-\!\bbPhi(\bbX;\! \hbS,\! \ccalB)\| \\
&\le\! \sum_{f=1}^{F_{L-1}}\!\big\|\bbx^f_{L\!-\!1} \!-\! \hat{\bbx}^f_{L\!-\!1} \big\| \!+\! 2C_L(1 \!+\! 8 \sqrt{N})\!\prod_{\ell = 1}^{L-1} F_\ell \big\| \bbX \big\| \eps \!+\! \ccalO^2(\eps).\nonumber
\end{align}
From \eqref{eqn:thm110}, we observe that the output difference of the $L$th layer depends on that of the $(L-1)$th layer. Repeating this recursion until the input layer, we have
 \begin{align} \label{eqn:thm111}
&\|\bbPhi(\bbX; \bbS, \ccalB)-\bbPhi(\bbX; \hbS, \ccalB)\| \\
&\le 2C_L(1 + 8 \sqrt{N}) L \prod_{\ell = 1}^{L-1} F_\ell \big\| \bbX \big\| \eps + \ccalO^2(\eps) \nonumber
\end{align}
where the initial condition $\| \bbx_0^1 - \hat{\bbx}_0^1 \| = \| \bbX - \bbX \|=0$ is used.

By substituting \eqref{eqn:thm11} and \eqref{eqn:thm111} into \eqref{eqn:thm100}, we get
 \begin{align}\label{eqn:thm112}
&\|\bbPsi(\bbX;\! \bbS, \ccalA,\! \ccalB)\!-\!\bbPsi(\bbX; \!\hbS, \ccalA,\! \ccalB)\| \\
&\le\! 2C_L(1 + 8 \sqrt{N}) \left(\! |\alpha_W| \!+\! |\alpha_D| L \prod_{\ell = 1}^{L-1} F_\ell \!\right)\! \big\| \bbX \big\| \eps \!+\! \ccalO^2(\eps)\nonumber
\end{align}
completing the proof.
\end{proof}

\section{Proof of Theorem 2} \label{Appendix:proofOfTheorem2}

\begin{proof}

Let $\ccalA^{\dag}$ and $\ccalB^{\dag}$ be the parameters learned from the offline phase. The proposed online learning procedure fixes the deep part, i.e., it freezes the parameters $\ccalB = \ccalB^{\dag}$, and retrains the wide part online. The model $\bbPsi(\bbX; \bbS, \ccalA, \ccalB)$ can then be represented as a function of the wide part parameters $\ccalA$ only, i.e., we have
 \begin{align}\label{proof:theorem1eq1}
\bbPsi(\bbX; \bbS, \ccalA, \ccalB^{\dag}) &= \alpha_{\text{W}} \bbA(\bbX; \bbS, \ccalA) + \alpha_{\text{D}} \bbPhi(\bbX; \bbS, \ccalB^{\dag}) + \beta \nonumber\\
& := \hat{\bbPsi}(\bbX; \bbS, \ccalA).
\end{align}
Given the graph signal $\bbX$ and the graph matrix $\bbS$, $\hat{\bbPsi}(\bbX; \bbS, \ccalA)$ is a linear function of $\ccalA$ since both the graph filter $\bbA(\bbX; \bbS, \ccalA)$ and the combination way of two components in \eqref{proof:theorem1eq1} are linear.

At testing time index $t$, the instantiated optimization problem \eqref{eq:timevaryingp} is translated to
 \begin{equation}\label{eq:timevaryingp1}
\min_{\ccalA} J_t\big(\hat{\bbPsi}(\bbX_t; \bbS_t, \ccalA)\big)
\end{equation}
where $J_t$, $\bbX_t$ and $\bbS_t$ are instantaneous loss function, observed testing signal and graph matrix at time index $t$. Since $\hat{\bbPsi}(\bbX_t; \bbS_t, \ccalA)$ is a linear function of $\ccalA$, $J_t\big(\hat{\bbPsi}(\bbX_t; \bbS_t, \ccalA)\big)$ is differentiable, strongly smooth with constant $C_{t,s}$ and strongly convex with constant $C_{t,c}$ based on Assumption \ref{asm:optimizationConvexity}.

We then denote by $J_t(\ccalA)=J_t\big(\hat{\bbPsi}(\bbX_t; \bbS_t, \ccalA)\big)$ the concise notation for convenience of theoretical derivations. From the centralized online learning update [cf. \eqref{eq:onlinelearning}], we have
 \begin{equation}\label{proof:theoremm21}
\ccalA_{t+1} = \ccalA_{t}-\gamma \nabla_{\ccalA} J_t\big(\ccalA_t\big).
\end{equation}
Subtracting $\ccalA_{t}^*$ into both sides of \eqref{proof:theoremm21} yields
 \begin{align}\label{proof:theoremm22}
	&\|\ccalA_{t+1} - \ccalA_t^*\| = \|\ccalA_{t}-\ccalA_t^* - \gamma \nabla_{\ccalA} J_t( \ccalA_t )\|
\end{align}
Since $\nabla_{\ccalA} J_t( \ccalA_t^* ) = 0$ due to the optimality of convex problems, we get
 \begin{align}\label{proof:theoremm23}
&\|\ccalA_{t+1} \!-\! \ccalA_t^*\| \!=\! \|\big(\ccalA_{t} \!-\! \gamma \nabla_{\ccalA} J_t( \ccalA_t )\big) \!-\! \big(\ccalA_t^* \!-\! \gamma \nabla_{\ccalA} J_t( \ccalA_t^* ) \big) \|. 
\end{align}
We now consider the operator $f_t(\ccalA) = \ccalA - \gamma \nabla_{\ccalA} J_t( \ccalA )$. If $J_t( \ccalA )$ is strongly smooth with constant $C_{t,s}$ and strongly convex with constant $C_{t,c}$ with $\gamma \in (0, 2 / C_{t,s}]$, the operator $f_t(\ccalA)$ is a contraction operator with the contraction factor $m_t = \max \{ |1-\gamma C_{t,s}|, |1-\gamma C_{t,c}| \}$ \cite{ryu2016primer}, i.e., we have
 \begin{align}\label{proof:theoremm24}
\|\ccalA_{t+1} \!-\! \ccalA_t^*\| \!=\! \|f(\ccalA_t) \!-\! f(\ccalA^*_t)\| \!\le\! m_t \|\ccalA_t - \ccalA_t^* \|
\end{align}

By adding and subtracting $\ccalA_{t}^*$ in $\|\ccalA_{t+1} - \ccalA_{t+1}^*\|$, we have
 \begin{align}\label{proof:theoremm25}
\|\ccalA_{t+1} - \ccalA_{t+1}^*\| &= \|\ccalA_{t+1} - \ccalA_{t}^* + \ccalA_{t}^* - \ccalA_{t+1}^*\| \\
& \le \|\ccalA_{t+1} - \ccalA_{t}^*\| + \|\ccalA_{t}^* - \ccalA_{t+1}^*\| \nonumber\\
&\le m_t \|\ccalA_t - \ccalA_t^* \| + \|\ccalA_{t}^* - \ccalA_{t+1}^*\|\nonumber 
\end{align}
where \eqref{proof:theoremm24} is used in the last inequality. By further using the fact $\|\ccalA_{t}^* - \ccalA_{t+1}^*\| \le C_B$ from Assumption \ref{asm:optimizationTimeVariation}, we get
 \begin{align}\label{proof:theoremm26}
&\|\ccalA_{t+1} - \ccalA_{t+1}^*\| \le m_t \|\ccalA_t - \ccalA_t^* \| + C_B.
\end{align}
Unrolling \eqref{proof:theoremm26} to the initial condition completes the proof
 \begin{equation}\label{proof:theoremm27}
\| \ccalA_{t+1} - \ccalA_{t+1}^* \| \le \big( \prod_{\tau=0}^{t} m_{\tau} \big) \| \ccalA_0 - \ccalA_0^* \| + \frac{1-\hat{m}^{t+1}}{1-\hat{m}} C_B
\end{equation}
where the formula of geometric series is applied and $\hat{m}= \max_{0 \le \tau \le t} m_\tau$.
\end{proof}

\section{Proof of Theorem 3}\label{Appendix:proofOfTheorem3}

\begin{proof}
	
we start by rewriting the distributed online algorithm [cf. \eqref{eq:disgd}] in a matrix form. Let $\bbd_{i,t} = \nabla_{\ccalA_i} J_{i,t}\big(\bbPsi(\bbX_t; \bbS_t, \ccalA_{i,t}, \ccalB^{\dag})\big)$ be the local gradient vector of node $i$ evaluated at $\ccalA_{i,t}$, and $\bbD_t = [\bbd_{1,t}, \ldots, \bbd_{N,t}]^{\Tr}$ be the matrix collecting these local gradients. Let also $\bbA_t = [\ccalA_{1,t}, \ldots, \ccalA_{N,t}]^{\Tr}$ be the matrix collecting the local parameters of all $N$ nodes\footnote{Note that the notation $\bbA_t = [\ccalA_{1,t}, \ldots, \ccalA_{N,t}]^{\Tr}$ is the matrix of local parameters in the proof of Theorem 3. It is different from the filter tap $\bbA_k$ in \eqref{eqn:graphConv}.}. The distributed online update can then be rewritten as
\begin{equation}
	\begin{aligned}\label{eq:disgdMatrix}
		\bbA_{t+1} &= \bbW_t \bbA_t - \gamma_t \bbD_t.
	\end{aligned}
\end{equation}
Following the recursive process, for any two time indices $t_1$ and $t_2+1$ satisfying $t_2 \ge t_1$, we have\footnote{Throughout this proof, we consider $\sum_{\tau=t_1}^{t_2} (\cdot) = 0$ if $t_2 < t_1$.}
\begin{equation}
	\begin{aligned}\label{eq:disgdMatrixMulti}
		\bbA_{t_2+\!1} \!=\!\!\! \prod_{\tau = t_1}^{t_2}\!\!\bbW_{\tau} \bbA_{t_1} \!-\!\!\!\!\!\sum_{\kappa = 0}^{t_2-t_1\!-\!1}\!\! \prod_{\tau=t_1\!+\!\kappa\!+\!1}^{t_2}\!\!\!\!\!\bbW_{\tau} \gamma_{t_1\!+\kappa} \bbD_{t_1\!+\kappa} \!-\! \gamma_{t_2}\bbD_{t_2}.
	\end{aligned}
\end{equation}
%
For a constant step-size $\gamma_t = \gamma$ and time indices $t_1=0, t_2 = t \ge 0$, \eqref{eq:disgdMatrixMulti} becomes
\begin{equation}
	\begin{aligned}\label{eq:disgdMatrixMulti0}
		\bbA_{t+1} &= \prod_{\tau = 0}^{t}\!\bbW_{\tau} \bbA_0 - \gamma\!\sum_{\kappa = 0}^{t-1} \prod_{\tau=\kappa+1}^{t}\!\bbW_{\tau} \bbD_{\kappa} \!-\! \gamma \bbD_t.
	\end{aligned}
\end{equation}
To ultimately show the convergence of $\bbA_t$, we first consider an intermediate scenario where nodes stop the distributed online update starting from the time index $T$ while continuing the local parameter aggregation with neighboring nodes, i.e., 
\begin{equation}
	\begin{aligned}\label{eq:Stopmodel}
		\bbD_{t} = \bb0,~\forall~t \ge T.
	\end{aligned}
\end{equation}
In this scenario, \eqref{eq:disgdMatrixMulti0} becomes
\begin{align}\label{eq:disgdMatrixStop}
	&\bbA_{t\!+\!1}(T) \\
	&\!=\! 
	\begin{cases}
		\prod_{\tau = 0}^{t}\!\!\bbW_{\tau} \bbA_0 \!-\! \gamma\!\sum_{\kappa = 0}^{T-1} \prod_{\tau=\kappa+1}^{t}\!\bbW_{\tau} \bbD_{\kappa},  \!\!&\!\! \mbox{if}~ t\!\ge\! T, \\
		\prod_{\tau = 0}^{t}\!\bbW_{\tau} \bbA_0 \!-\! \gamma\!\sum_{\kappa = 0}^{t-1} \prod_{\tau=\kappa+1}^{t}\!\bbW_{\tau} \bbD_{\kappa} - \gamma \bbD_t, \!\!&\!\! \mbox{if}~ t \!<\! T.\nonumber
	\end{cases} \nonumber
\end{align}
From Lemma \ref{lemma:limitMatrix}, we know that $\lim_{t \to \infty}\prod_{\tau = 0}^{t}\!\bbW_{\tau} = \bbE/N$ with $\bbE = \bbone \bbone^{\Tr}$. We then define the limiting matrix of $\bbA_t(T)$ as $t$ goes to the infinity as
\begin{equation}
	\begin{aligned}\label{eq:disgdMatrixStopLimit}
		\bbC(T):=\lim_{t \to \infty}\bbA_t(T) &= \frac{1}{N} \bbE \bbA_0 - \frac{\gamma}{N}\sum_{\kappa = 0}^{T-1} \bbE \bbD_{\kappa},
	\end{aligned}
\end{equation}
and obtain the relation between $\bbC(T)$ and $\bbC(T+1)$ as
\begin{equation}
	\begin{aligned}\label{eq:disgdMatrixStopRecursive}
		\bbC(T+1) = \bbC(T) - \frac{\gamma}{N} \bbE \bbD_{T}.
	\end{aligned}
\end{equation}
Separating $\bbC(T)$ to each local parameter $\ccalC_i(T)$ yields
\begin{equation}
	\begin{aligned}\label{eq:disgdMatrixStopRecursiveNode}
		\ccalC_{i}(T+1) = \ccalC_i(T) - \frac{\gamma}{N} \sum_{j=1}^N \bbd_{j,T}
	\end{aligned}
\end{equation}
for $i=1,\ldots,N$. Since the same equation \eqref{eq:disgdMatrixStopRecursiveNode} holds for all local parameters $\ccalC_1(T),\ldots,\ccalC_N(T)$, we denote by $\ccalC(T)$ a uniform notation to represent any one of local parameters. Following the recursion in \eqref{eq:disgdMatrixStopRecursiveNode}, we obtain
\begin{equation}
\begin{aligned}\label{proof:thm31}
\ccalC(T) = \ccalC(0) - \frac{\gamma}{N}\sum_{\kappa = 0}^{T-1} \sum_{i=1}^N \bbd_{i,\kappa}
\end{aligned}
\end{equation}
for all $T \ge 1$ with $\ccalC(0) = \sum_{i=1}^N \ccalA_{i,0}/N$. By comparing \eqref{proof:thm31} with the distributed online update [cf. \eqref{eq:disgd}], we have
\begin{align}\label{proof:thm32}
\|\ccalC(T) - \ccalA_{i,T}\| &\le \Big\| \sum_{j=1}^N \Big( \frac{1}{N} - [\bbLambda(0,T-1)]_{j}^i \Big) \ccalA_{j,0} \\
& \!-\!\gamma\! \sum_{\kappa =0}^{T-2}\! \sum_{j\!=\!1}^N\! \Big(\! \frac{1}{N} \!-\! [\bbLambda(\kappa+1,T\!-\!1)]_{j}^i \Big) \bbd_{j,\kappa} \nonumber\\
&\!-\! \gamma \Big( \frac{1}{N}\! \sum_{j=1}^N\! \bbd_{j,T\!-\!1} \!-\! \bbd_{i,T-1} \!\Big)\Big\|\nonumber
\end{align}
where $\bbLambda(0,T-1) = \prod_{t = 0}^{T-1}\bbW_{t}$ is the production of a sequence of weighted matrices, so as $\bbLambda(\kappa+1,T\!-\!1)$. The triangle inequality allows us to bound \eqref{proof:thm32} as
\begin{align}\label{proof:thm33}
\|\ccalC(T) - \ccalA_{i,T}\| \!&\le\! \max_{1 \le j \le N}\| \ccalA_{j,0} \| \sum_{j=1}^N \Big| \frac{1}{N} \!-\! [\bbLambda(0,T\!-\!1)]_{j}^i \Big| \nonumber \\
& +\!\gamma\! \sum_{\kappa =0}^{T-2} \sum_{j=1}^N \| \bbd_{j,\kappa} \| \Big| \frac{1}{N} \!-\! [\bbLambda(\kappa+1,T-1)]_{j}^i \Big|  \nonumber \\
&+\frac{\gamma}{N} \sum_{j=1}^N \Big\| \bbd_{j,T-1} - \bbd_{i,T-1} \Big\|.
\end{align}
By using Lemma \ref{lemma:limitMatrix} [cf. \eqref{lemmaeq:Matrixbound}] in the first two terms of \eqref{proof:thm33}, the triangle inequality in the third term of \eqref{proof:thm33}, $\| \bbd_{i,T} \| \le L$ for all $T \ge 0$ from Assumption \ref{asm:optimizationConvexity}, and the initial condition $\| \ccalA_{j,0} \| \le \gamma L$, we get 
\begin{align}\label{proof:thm34}
&\|\ccalC(T) \!-\! \ccalA_{i,T}\| \!\le\! 2 \gamma L \!\left(\!\! 1 \!+\! \frac{N}{1\!\!-\!\!(1\!-\!\epsilon^{\hat{C_d}})^{1/\hat{C_d}}} \frac{1\!+\!\epsilon^{-\hat{C_d}}}{1\!-\!\epsilon^{\hat{C_d}}}\! \right)\!.
\end{align}
with $\hat{C_d} = C_d (N-1)$. By substituting \eqref{proof:thm34} into Lemma \ref{Lemma:iterationRelation} [cf. \eqref{lemmaeq:iterationRelation}], we have
\begin{align}\label{proof:thm35}
&\|\ccalC(T\!+\!1) \!-\! \ccalA^*_{T+1}\| \!\le\! m_T \|\ccalC(T) \!-\! \ccalA^*_T\| \!+\!2 \gamma^2 C_{T,s} L C_\epsilon \!+\! C_B
\end{align}
where 
\begin{equation}
\begin{aligned}\label{proof:thm36}
C_\epsilon = 1 + \frac{N}{1-(1-\epsilon^{\hat{C_d}})^{1/\hat{C_d}}} \frac{1+\epsilon^{-\hat{C_d}}}{1-\epsilon^{\hat{C_d}}}
\end{aligned}
\end{equation}
is a constant. Unrolling this recursion, we obtain
\begin{equation}
\begin{aligned}\label{proof:thm37}
&\|\ccalC(T+1) - \ccalA^*_{T+1}\| \\
& \le \prod_{i=0}^T m_i \|\ccalC(0) - \ccalA^*_0\| \!+\!\frac{1-\hat{m}^T}{1-\hat{m}}\big( 2 \gamma^2 C_s L C_\epsilon + C_B \big)
\end{aligned}
\end{equation}
with $\hat{m} = \max_{0\le t \le T} m_t$ and $C_s = \max_{0\le t \le T} C_{t,s}$. Again using \eqref{proof:thm34} in \eqref{proof:thm37} with the triangle inequality and the fact $\ccalC(0)=1/N \sum_{i=1}^N \ccalA_{i,0}$, we get
\begin{align}\label{proof:thm38}
&\|\ccalA_{i,T+1} - \ccalA^*_{T+1}\| \\
& \le \prod_{i=0}^T m_i \Big\|\frac{1}{N}\sum_{i=1}^N \ccalA_{i,0} - \ccalA^*_0\Big\| \!+\!\frac{2 \gamma^2 C_s L C_\epsilon + C_B}{1-\hat{m}} + 2 \gamma L C_\epsilon \nonumber
\end{align}
completing the proof.

\end{proof}

\begin{lemma} \label{lemma1}
	Consider the underlying graph $\bbS = \bbV \bbLambda \bbV^\top$ and the perturbated graph $\hat{\bbS}$ with $N$ nodes. The relative perturbation $\bbE= \bbU \bf{\Theta}\bbU^\top \in \ccalR $ satisfies $d(\bbS, \hbS) \le \| \bbE \| \le \eps$. Consider the integral Lipschitz filter [Def. \ref{def:ILfilters}] with $F=1$ input feature, $G=1$ output feature and integral Lipschitz constant $C_L$. Then, the output difference between filters $\bbA(\bbX; \bbS, \ccalA)$ and $\bbA(\bbX; \hat{\bbS}, \ccalA)$ satisfies
	\begin{equation} \label{eq:lemma10}
		\| \bbA(\bbX; \bbS, \ccalA) - \bbA(\bbX; \hat{\bbS}, \ccalA)\| \le 2C_L(1 + \sigma \sqrt{N})\| \bbX \| \eps + \ccalO(\eps^2)
	\end{equation}
	where $\delta = (\| \bbU - \bbV \|+1)^2-1$ implies the eigenvector misalignment between $\bbS$ and $\bbE$.
\end{lemma} 
\begin{proof}
	With $F=1$ input feature $\bbX \in \mathbb{R}^{N \times 1}$ and $G=1$ output features $\bbA(\bbX; \bbS, \ccalA), \bbA(\bbX; \hat{\bbS}, \ccalA) \in \mathbb{R}^{N \times 1}$, the output difference can be represented by
	\begin{equation} \label{eq:lemma11}
		\| \bbA(\bbX;\! \bbS,\! \ccalA) -\! \bbA(\bbX;\! \hat{\bbS},\! \ccalA)\| \!=\! \Big\| \sum_{k=0}^K \!a_k \bbS^k \bbX \!-\! \sum_{k=0}^K\! a_k \hbS^k \bbX \Big\|
	\end{equation}
	with filter parameters $\ccalA = \{ a_0,\ldots,a_K \}$ [cf. \eqref{eqn:graphConv}]. We then refer to Theorem 2 in \cite{Gama2020-Stability} to complete the proof.
\end{proof}

%
\begin{lemma}\label{lemma:limitMatrix}
	Consider the time-varying graph $\ccalG_t$ satisfying Assumption \ref{as:Connectivity}, Assumption \ref{as:Interval} with constant $C_d$ and Assumption \ref{as:Doublystochastic} with constant $\epsilon$. Let $\bbW_{t}$ be the weighted matrix of the distributed online learning update at time index $t$ [cf. \eqref{eq:disgd}] and $\bbLambda(t_1,t_2) = \prod_{t = t_1}^{t_2}\bbW_{t}$ be the production of weighted matrices from time index $t_1$ to time index $t_2$. Then, for any $i,j \in \{ 1, \ldots, N \}$, the entry $[\bbLambda(t_1,t_2)]_{ij}$ converges to $1/N$ as $t_2 \to \infty$ at a geometric rate, i.e.,
	\begin{equation}
		\begin{aligned}\label{lemmaeq:Matrixbound}
			\Big| [\bbLambda(t_1,t_2)]_{ij} - \frac{1}{N} \Big| \le 2 \frac{1+\epsilon^{-\hat{C_d}}}{1-\epsilon^{\hat{C_d}}}\big( 1-\epsilon^{\hat{C_d}} \big)^{(t_2-t_1)/\hat{C_d}}
		\end{aligned}
	\end{equation}
	with $\hat{C_d} = C_d (N-1)$. Moreover, let $\bbLambda(t_1) = \lim_{t_2 \to \infty} \bbLambda(t_1,t_2)$ be the limiting matrix and then, it holds that
	\begin{equation}
		\begin{aligned}\label{lemmaeq:MatrixLimit}
			\bbLambda(t_1) = \frac{1}{N} \bbone \bbone^{\Tr}
		\end{aligned}
	\end{equation}
	where $\bbone=[1,\ldots,1]^{\Tr} \in \mathbb{R}^N$ is a vector of ones.
\end{lemma}
\begin{proof}
	See Proposition 1 in\cite{nedic2009distributed}. 
\end{proof}


\begin{lemma}\label{Lemma:iterationRelation}
	Consider the WD-GNN [cf. \eqref{eqn:WDGNN}] optimized with the distributed online learning procedure [cf. \eqref{eq:disgd}]. Let $J_{i,T}\big(\bbPsi(\bbX_T; \bbS_T, \ccalA_i, \ccalB^{\dag})\big)$ be the time-varying local loss function satisfying Assumptions \ref{asm:optimizationTimeVariation}-\ref{asm:optimizationConvexity} with constants $C_B$, $C_{T,s}$ and $C_{T,c}$ for $i=1,\ldots,N$, $\ccalA^*_T$ be the optimal solution of $\sum_{i=1}^N J_{i,T}\big(\bbPsi(\bbX_T; \bbS_T, \ccalA, \ccalB^{\dag})\big)/N$, and $\gamma_T = \gamma \in (0,2/C_{T,s}]$ be the constant step-size. Let also $\{ \ccalC(T) \}_T$ be the sequence generated by \eqref{eq:disgdMatrixStopRecursiveNode}, $\{ \ccalA_{i,T} \}_{i=1}^N$ be the local parameters generated by \eqref{eq:disgd} at time index $T$, and $\bbg_{i,T}$ be the gradient of local loss function evaluated at $\ccalC(T)$, i.e., $\bbg_{i,T} = \nabla_{\ccalA_i} J_{i,T}\big(\bbPsi(\bbX_T; \bbS_T, \ccalC(T), \ccalB^{\dag})\big)$ for all $i=1,\ldots, N$. Then, it holds that
	\begin{align}\label{lemmaeq:iterationRelation}
			&\|\ccalC(T+1) - \ccalA^*_{T+1}\| \\
			& \le m_T \|\ccalC(T) - \ccalA^*_{T}\| \!+\!\frac{\gamma C_{T,s}}{N} \sum_{i=1}^N \| \ccalC(T) - \ccalA_{i,T} \| + C_B \nonumber
	\end{align}
	where $m_T\!=\!\max\{ |1-\gamma C_{T,s}|, |1-\gamma C_{T,c}| \}$ is the convergence rate.
\end{lemma}

\begin{proof}
	By subtracting $\ccalA^*_{T+1}$ in both sides of \eqref{eq:disgdMatrixStopRecursiveNode}, we have
	\begin{align}\label{proof:prop10}
		&\|\ccalC(T+1) - \ccalA^*_{T+1}\| = \Big\| \ccalC(T) \!-\! \ccalA^*_{T+1} - \frac{\gamma}{N} \sum_{i=1}^N \bbd_{i,T}\Big\|
	\end{align}
	where $\bbd_{i,T} = \nabla_{\ccalA_i} J_{i,T}\big(\bbPsi(\bbX_T; \bbS_T, \ccalA_{i,T}, \ccalB^{\dag})\big)$ is the gradient of local loss function evaluated at $\ccalA_{i,t}$. Adding and subtracting $\gamma \sum_{i=1}^N \bbg_{i,T}/N$ in the right side of \eqref{proof:prop10} yields
	\begin{align}\label{proof:prop11}
		&\|\ccalC(T+1) - \ccalA^*_{T+1}\| \\
		&=\! \Big\| \ccalC(T) \!-\! \ccalA^*_{T+1} \!-\! \frac{\gamma}{N} \sum_{i=1}^N \bbg_{i,T} \!+\!\frac{\gamma}{N} \sum_{i=1}^N \bbg_{i,T} \!-\! \frac{\gamma}{N} \sum_{i=1}^N \bbd_{i,T}\Big\| \nonumber\\
		&\le\! \Big\| \ccalC(T) \!-\! \ccalA^*_{T+1} \!-\! \frac{\gamma}{N} \!\sum_{i=1}^N \bbg_{i,T}\Big\| \!+\!\Big\|\frac{\gamma}{N}\! \sum_{j=1}^N \!\bbg_{i,T} \!-\! \frac{\gamma}{N}\! \sum_{i=1}^N \!\bbd_{i,T}\Big\| \nonumber
	\end{align}
	where the triangle inequality is used. We consider two terms in the bound of \eqref{proof:prop11} separately. For the first term $\big\| \ccalC(T) - \ccalA^*_{T+1} \!-\! \gamma \sum_{i=1}^N \bbg_{i,T}/N \big\|$ in \eqref{proof:prop11}, note that $\ccalC(T) - \gamma \sum_{i=1}^N \bbg_{i,T}/N$ is a gradient descent update of the cost function $\sum_{i=1}^N J_{i,T}\big(\bbPsi(\bbX_T; \bbS_T, \ccalA, \ccalB^{\dag})\big)/N$ at the decision variable $\ccalC(T)$ with the step-size $\gamma$. From \eqref{proof:theoremm26} in the proof of Theorem \ref{thm:convergence}, we know that
	\begin{align}\label{proof:prop12}
		&\Big\| \ccalC(T) \!-\! \frac{\gamma}{N} \sum_{i=1}^N \bbg_{i,T} \!-\! \ccalA^*_{T+1} \Big\| \!\le\! m_T \|\ccalC(T) \!-\! \ccalA^*_T\| \!+\! C_B
	\end{align}
	where $m_T = \max\{ |1-\gamma C_{T,s}|,|1-\gamma C_{T,c}| \}$ is the convergence rate. For the second term $\|\gamma \sum_{i=1}^N \bbg_{i,T}/N - \gamma \sum_{i=1}^N \bbd_{i,T}/N\|$ in \eqref{proof:prop11}, we can bound it with the triangular inequality as
	\begin{equation}
		\begin{aligned}\label{proof:prop13}
			&\Big\|\frac{\gamma}{N} \sum_{i=1}^N \bbg_{i,T} - \frac{\gamma}{N} \sum_{i=1}^N \bbd_{i,T}\Big\| \le \frac{\gamma}{N}\! \sum_{i=1}^N \|  \bbg_{i,T} \!-\!  \bbd_{i,T}\|.
		\end{aligned}
	\end{equation}
	From the strongly smoothness in Assumption \ref{asm:optimizationConvexity}, we have
	\begin{equation}
		\begin{aligned}\label{proof:prop14}
			& \|  \bbg_{i,T} -  \bbd_{i,T}\| \le C_{T,s} \| \ccalC(T) - \ccalA_{i,T} \|.
		\end{aligned}
	\end{equation}
	By substituting \eqref{proof:prop14} into \eqref{proof:prop13} and altogether into \eqref{proof:prop11} with \eqref{proof:prop12}, we get
	\begin{align}\label{proof:prop15}
			&\|\ccalC(T+1) - \ccalA^*_{T+1}\| \\
			& \le m_T \|\ccalC(T) - \ccalA^*_T\| + \frac{\gamma C_{T,s}}{N} \sum_{i=1}^N \| \ccalC(T) - \ccalA_{i,T} \| + C_B \nonumber
		\end{align}
	completing the proof.
\end{proof}



\bibliographystyle{bibFiles/IEEEtranD}
\bibliography{bibFiles/myIEEEabrv,bibFiles/biblioWideDeep}

\end{document}